\documentclass{article}
\usepackage{graphicx}

\usepackage{comment}
\usepackage{amsmath,amsthm,amsfonts,amssymb,amscd}
\usepackage{authblk}
\usepackage{geometry}
\usepackage{mathtools}
\usepackage{float}
\usepackage{graphicx}
\usepackage{enumitem}
\usepackage{tikz}
\usepackage[square,sort,comma,numbers]{natbib}
\usepackage{subcaption}
\usepackage{hyperref}
\usepackage{thmtools} 
\usepackage{thm-restate}

\usepackage{mathrsfs}
\hypersetup{
     colorlinks=true,
     linkcolor=blue,
     filecolor=blue,
     citecolor = black,      
     urlcolor=cyan,
     }
\usepackage[T1]{fontenc}
\newtheorem{remark}{Remark}[section]
\newtheorem{theorem}{Theorem}[section]
\newtheorem{proposition}{Proposition}[section]

\newtheorem{lemma}{Lemma}[section]
\newtheorem{claim}{Claim}[section]

\theoremstyle{definition}
\newtheorem{definition}{Definition}[section]

\newcommand{\R}{\mathbb{R}}
\newcommand{\N}{\mathbb{N}}

\newcommand{\Z}{\mathbb{Z}}

\newcommand{\E}{\mathbb{E}}
\newcommand{\W}{\mathcal{W}}
\newcommand{\D}{\mathcal{D}}
\newcommand{\K}{\mathcal{K}}

\newcommand{\vc}{\vcentcolon}
\newcommand{\eps}{\varepsilon}
\newcommand{\ind}{\mathbf{1}}

\DeclareMathOperator{\Ima}{Im}

\def\S{{\mathbb S}}

\def\P{{\mathbb P}}
\def\M{{\mathcal M}}
\newcommand{\grad}{\mathrm{grad}}

\usepackage{amsmath}

\usepackage[capitalise]{cleveref}
\usepackage{todonotes}

\title{On the convergence of graph Laplacians with a symmetric divergence}
\author{Liane Xu\thanks{Program in Applied and Computational Mathematics, Princeton University}}
\date{\vspace{-7.5ex}}
\begin{document}
\maketitle
\begin{abstract}
   When analyzing a manifold learning algorithm for data lying on a smooth, compact, connected Riemannian submanifold $(\M, g)$ of $\R^d$, a key estimate for the geodesic distance $d_g$ is that there exists $K > 0$ such that $0 \leq d_g(p, q)^2 - \|p-q\|^2  \leq K d_g(p, q)^4$ for all $p, q \in \M$.
    We observe that more generally, when $\M$ is equipped with a smooth symmetric divergence $D$ satisfying a non-degeneracy condition and $g$ is given by $g_p \vc= \frac{1}{2}\mathrm{Hess}_p(D(p, \cdot))$ for all $p \in \M$, there exists $K > 0$ such that
    $\left| D(p, q) - d_g(p, q)^2 \right| \leq K d_g(p, q)^4$
    for all $p, q \in \M$. We demonstrate that this is sufficient for the pointwise convergence of graph Laplacians constructed with $D$ and discuss examples where $D$ is given by the Sinkhorn divergence on a family of probability measures parametrized by a manifold.
\end{abstract}

\let\thefootnote\relax

\tableofcontents

\section{Introduction}

\subsection{Motivation}
A popular assumption in data science is that high-dimensional data lie on or near a lower-dimensional manifold $\M$ \cite{fefferman2016testing, tenenbaum2000global, belkin2003laplacian}.
However, in certain applications — for example, when trying to understand the different conformations that a molecule can take \cite{zelesko2020earthmover} — we may want to model the underlying data generation process using a family of probability measures $\{ \mu_x \colon x\in \M\} \subset \mathcal{P}(X)$ supported on some compact $X \subset \R^d$, parametrized by an unknown manifold $\M$.
Specifically, suppose we are given $\mu_{x_1}, \ldots, \mu_{x_N}$, where $x_1, \ldots, x_N$ drawn i.i.d. from a probability distribution on $\M$. 
Under what assumptions can we learn something about $\M$, and how can we learn it?

One option is to try to embed $\mu_{x_1}, \ldots, \mu_{x_N}$ into a low-dimensional Euclidean space using a manifold learning algorithm such as Isomap \cite{tenenbaum2000global}, Laplacian eigenmaps \cite{belkin2001laplacian, belkin2003laplacian,belkin2005towards} or diffusion maps \cite{coifman2006diffusion}.
More concretely, each of the aforementioned algorithms only requires a measure of dissimilarity to compare data points. For data that lie in a Euclidean space, the Euclidean distance is often used. For our problem, a natural approach would be to use one of these algorithms, but to compare $\mu_{x_i},\mu_{x_j}$ for each $i, j \in [N]$ using a metric or divergence defined on probability measures. Depending on the application (e.g., \cite{lu2023discovering, zelesko2020earthmover}), one may want to use a metric arising from optimal transport, such as the Wasserstein-2 distance. Manifold learning with the Wasserstein-2 distance 
has also received interest \cite{hamm2024manifoldlearningwassersteinspace, oliver2025laplace} because the Wasserstein-2 space has a formal Riemannian structure \cite{otto2001geometry}.
However, since this structure is only formal, it is much more difficult to work with than a bona fide finite-dimensional smooth Riemannian manifold. 

Perhaps unsurprisingly, regularity assumptions can simplify the analysis of manifold learning algorithms. A key estimate for a smooth, connected, compact Riemannian submanifold $(\M, g)$ of $\R^d$ is that there exists $K > 0$ such that
\begin{equation}
\label{eq:euclidean-estimate-intro}
\|p-q\|^2 \leq d_g(p, q)^2 \leq \|p-q\|^2 + Kd_g(p, q)^4
\end{equation}
for all $p, q \in \M$, where $d_g$ denotes the geodesic distance on $(\M, g)$ \cite{belkin2008towards,bernstein2000graph}. Now suppose that we are instead only given a smooth, compact, connected, orientable manifold $\M$, not necessarily embedded in Euclidean space. In previous work, we found that if $\M$ is embedded in a metric space $(X, d)$, $d^2$ is smooth on a neighborhood of the diagonal $\Xi_\M$ of $\M \times \M$ and 
\[g_p \vc = \frac{1}{2}\mathrm{Hess}_p d^2(p, \cdot)\] is non-degenerate for all $p \in \M$, then a similar inequality to (\ref{eq:euclidean-estimate-intro}) holds, with the Euclidean distance replaced by $d$ \cite{xu2025manifold}.
Using this, we can prove the pointwise convergence of graph Laplacians — a key object in the Laplacian eigenmap and diffusion map algorithms — following a similar argument as the Euclidean setting \cite{belkin2005towards, belkin2008towards}. Separately, Arias-Castro and Qiao obtained results for Isomap under slightly different conditions \cite{arias2025embedding}; although the analysis of Isomap does not require the fourth-order error in (\ref{eq:euclidean-estimate-intro}) \cite{bernstein2000graph}, as noted in \cite{arias2025embedding}, assuming some regularity of $d^2$ around the diagonal $\Xi_\M$ does simplify the analysis.

This current work is motivated by the problem of understanding manifold learning algorithms when the Sinkhorn divergence is used to compare $\mu_{x_i},\mu_{x_j}$ for each $i, j \in [N]$. The Sinkhorn divergence, a debiased version of entropy-regularized optimal transport \cite{genevay2018learning, feydy2019interpolating}, is popular for several reasons. For one, it can be approximated quickly via Sinkhorn's algorithm \cite{sinkhorn1964relationship, cuturi2013sinkhorn}. In addition, for compact $X$, it is differentiable on $\mathcal{P}(X) \times \mathcal{P}(X)$ and convex with respect to each of its inputs \cite{feydy2019interpolating}, and its sample complexity is given by the parametric rate (as opposed to that of the squared Wasserstein-2 distance, which suffers from the curse of dimensionality) \cite{genevay2019sample, mena2019statistical}.
However, since the Sinkhorn divergence itself is not the square of a metric \cite{Lavenant2025}, we cannot apply our previous results from \cite{xu2025manifold}.

\subsection{Contributions}
The contributions of this paper are twofold.
First, we observe that for a smooth, symmetric divergence $D \colon \M \times \M \to [0, \infty)$ under a non-degeneracy condition (see Assumption \ref{assump:divergence} for details), we can use the symmetry of $D$ to obtain the estimate
\begin{equation}
\label{eq:intro-4th-order}
\left| D(p, q) - d_g(p, q)^2 \right| \leq K d_g(p, q)^4, 
\end{equation}
for all $p, q \in \M$, where $K > 0$ is a constant uniform across all $p, q \in \M$ and $d_g$ is the geodesic distance on $(\M, g)$ with $g_p \vc= \frac{1}{2}\mathrm{Hess}_p(D(p, \cdot))$.
From here, the pointwise convergence of graph Laplacians constructed with $D$ can be proven \emph{mutatis mutandis} from the case when $(\M, g)$ is a compact Riemannian submanifold of Euclidean space and the graph Laplacians are constructed with the squared Euclidean distance \cite{belkin2005towards, belkin2008towards, coifman2006diffusion}.

Secondly, since the Sinkhorn divergence has nice regularity properties in several situations of interest, we can apply our results to these situations. 
Specifically, in \cref{sec:sink-div}, we will study the regularity of the Sinkhorn divergence and the non-degeneracy condition under the following two scenarios:
\begin{restatable*}[Particle system parametrized by a manifold]{assumption}{manifoldhyp}
    \label{assump:manifold-hyp}
    Let $\M$ be a smooth, connected, compact orientable $m$-dimensional manifold with $m \geq 1$.
    Assume that we have a smooth embedding
     $\iota = (\iota_1, \cdots, \iota_n) \colon \M \to (\R^d)^n$ such that $\iota_j(x) \neq \iota_k(x)$ for all $x \in \M$ and $j \neq k$.
Fix some $p_1, \ldots, p_n > 0$ such that $\sum_{j = 1}^n p_j = 1$, and define $F \colon \M \to \mathcal{P}(\R^d)$ by
\begin{align*}
    F \colon \M &\to \mathcal{P}(\R^d) \\
    x& \mapsto \sum_{j = 1}^n p_j \delta_{\iota_j(x)}.
\end{align*}
Moreover, assume that $F(x) \neq F(y)$ for all $x \neq y$.
\end{restatable*}

\begin{restatable*}[Deformations of a measure]{assumption}{manifoldhypdeform}
    \label{assump:manifold-hyp-deform}
    Fix any $k \in \Z_{> 0}$.
    Let $X = \overline{\Omega}$ be the closure of a bounded open convex subset
    $\Omega \subset \R^d$, and let $U \subset \R^m$ be an open set.
Let $\mu \in \mathcal{P}(X)$. Suppose $\Psi \colon U \times X \to X$ is a continuous function and that its first $k$ partial derivatives with respect to the first variable
\begin{equation}
\label{eq:psi-part-derivs}
\frac{\partial^k \Psi}{\partial p_{i_1} \partial p_{i_2} \cdots \partial p_{i_k}} \colon U \times X \to \R^d 
\end{equation}
exist and are continuous on $U \times X$ for all $i_1, \cdots, i_k \in [m]$.
Define
\begin{align*}
    F \colon U &\to \mathcal{P}(X) \\
    p &\mapsto \Psi(p, \cdot)_{\#}\mu.
\end{align*}
We will sometimes also write $\mu_p$ for $F(p) = \Psi(p, \cdot)_{\#}\mu.$ Assume that $\mu_p \neq \mu_q$ for all $p \neq q$.
\end{restatable*}

Both assumptions are motivated in part by examples in existing literature.
For example, one can construct an idealized model of a protein by treating each atom as a Dirac delta or point mass (e.g., \cite{singer2021wilson, diepeveen2024riemannian} and references therein), upon which Assumption \ref{assump:manifold-hyp} can be viewed as an idealized model for a protein whose conformation space is parametrized by a manifold. Although two structures are often considered equivalent if one can be transformed into the other via rotations and translations (see \cite{diepeveen2024riemannian} and references therein for ways to handle this), we will not consider this issue in this paper.
Assumption \ref{assump:manifold-hyp-deform} is similar to 
\cite[Section 2.2]{hamm2024manifoldlearningwassersteinspace}, the "horizontal perturbations" in \cite[Remark 3.2]{Lavenant2025} and the assumptions of \cite{shen2020sinkhorn} — each considers a family of measures defined by deforming a template measure — though we use different regularity conditions.
We start with Assumption \ref{assump:manifold-hyp}, as under this assumption, we can illustrate the main ideas without functional analysis. The argument under Assumption \ref{assump:manifold-hyp-deform} is similar but involves more analysis.

The rest of this paper is organized as follows. In \cref{sec:background}, we provide some background on graph Laplacians and Sinkhorn divergences and discuss related work. In \cref{sec:gl-conv}, we discuss (\ref{eq:intro-4th-order}), i.e., that one can approximate the induced squared geodesic distance with the given symmetric divergence with a fourth-order error, and the pointwise convergence of graph Laplacians.
In \cref{sec:sink-div}, we focus more specifically on the regularity of the Sinkhorn divergence and the non-degeneracy condition under Assumptions \ref{assump:manifold-hyp} and \ref{assump:manifold-hyp-deform}.
Lastly, we study two examples in \cref{sec:examples} and conclude in \cref{sec:concl}.

\section{Preliminaries}
\label{sec:background}

Throughout this paper, unless otherwise specified, smooth means $C^\infty$, and we will follow the notation and assumptions in Table \ref{tab:notation}.
When working with a smooth Riemannian manifold $(\M, g)$, we assume that it is equipped with its Levi-Civita connection and that the Laplace-Beltrami operator is given by $\Delta_g f \vc = -\mathrm{div}_g(\grad_g f)$ for $f \in C^2(\M)$.

\begin{table}[ht]
    \centering
    \caption{Notation and assumptions for this paper.}
    \label{tab:notation}
    
    \begin{tabular}{| c | c |}
        \hline
        $\M$   &   smooth, connected, compact orientable $m$-dimensional manifold, $m \geq 1$   \\
        $m$   & dimension of $\M$    \\
        $h$ & function satisfying Assumption \ref{assump:lap-eig-kernel}, used in the construction of the graph Laplacian \\
        $D$ & divergence satisfying Assumption \ref{assump:divergence} \\
        $\eps$ & bandwidth parameter for the graph Laplacian; see \cref{subsec:gl-construction} \\
        $N$ & number of samples \\
        $L^{(\eps, N)}$ & discrete Laplacian, defined in \cref{subsec:gl-construction} \\ 
        $\beta$ & regularization parameter for the Sinkhorn divergence; see \cref{subsec:background-sink-div} \\
        $\mathrm{OT}_\beta$ & optimal transport with entropy regularization $\beta > 0$; see (\ref{eq:ent-reg-ot}) \\
        $S_\beta$ & Sinkhorn divergence with regularization $\beta > 0$; see (\ref{eq:def-sinkdiv}) \\
        $\mathcal{P}(X)$ & the space of Borel probability measures on $X \subset \R^d$ \\
        \hline
    \end{tabular}
\end{table}

\subsection{Construction of the graph Laplacian}
\label{subsec:gl-construction}
We begin by describing the construction of the discrete Laplacians that we will use throughout this paper. 
Fix any $\eps > 0$. We assume that we have functions $h \colon [0, \infty) \to [0, \infty)$ and $D \colon \M \times \M \to \R$ satisfying the following conditions:
\begin{restatable}[Assumptions on $h$ (following \cite{cheng2022convergence, cheng2022eigen})]{assumption}{lap-eig-kernel}
    \label{assump:lap-eig-kernel}
    Let $h \colon [0, \infty) \to [0, \infty)$ be a continuous function such that $h \in C^2((0, \infty))$ and that $h, h', h''$ are exponentially decaying, i.e., there exist $c, c_0, c_1, c_2 > 0$ such that
    \[ \left| h^{(k)}(x) \right| \leq c_k e^{-cx} \]
    for all $x > 0$ and $k = 0, 1, 2.$
    Assume also that $h$ is not identically $0$.
\end{restatable}

\begin{restatable}[Assumptions for the divergence $D$]{assumption}{divergence}
    \label{assump:divergence}
    Let $D \colon \M \times \M \to \R$ be a continuous function satisfying the following conditions:
    \begin{enumerate}
        \item $D(p, q) \geq 0$ for all $p, q \in \M$;
        \item $D(p, q) = 0$ if and only if $p = q$;
        \item $D(p, q) = D(q, p)$ for all $p, q \in \M$;
        \item $D$ is smooth 
        on an open neighborhood of the diagonal $\Xi_\M \vc = \{(p, p) \colon p \in \M \} \subset \M \times \M$;
        \item For all $p \in \M$,
        \begin{equation}
        \label{eq:hess-riem-metric}
            g_p \vc = \frac{1}{2}\mathrm{Hess}_p (D(p, \cdot))
        \end{equation}
        is positive definite, so $(\M, g)$ is a Riemannian manifold. Although the Hessian typically depends on a Riemannian metric, for each $p \in \M$, $f(x) \vc = D(p, x)$ attains a minimum at $p$, so for all $V, W \in T_p\M$,
        \[ \mathrm{Hess}_p f(V, W) = V(Wf) \]
        is well-defined independent of any Riemannian metric on $\M$.
    \end{enumerate}
\end{restatable}
\begin{remark}
In information geometry, for a smooth manifold $\M$, a smooth function $D \colon \M \times \M \to \R$ satisfying 1., 2., and 5. in Assumption \ref{assump:divergence} above is known as a \textit{divergence}; symmetry (condition 4) is not required \cite[Definition 1.1]{amari2016information}. 
\end{remark}

Now suppose that $x_1, \ldots, x_N$ are drawn i.i.d. from a Borel probability measure $\P$ on $\M$. Define $\K_\eps \colon \M \times \M \to [0, \infty)$ via
\begin{equation}
\label{eq:lap-eig-kernel-def}
    \K_\eps(x, y) \vc = h \left( \frac{D(x, y)}{\eps} \right)
\end{equation}
for each $x, y \in \M$. We now define the graph Laplacians that we will use:
\begin{definition}[Graph Laplacian]
Construct a graph with $\{x_i\}_{i \in [N]}$ as vertices, and set the edge weight from $x_i$ to $x_j$ to be
\begin{equation}
\label{eq:weights}
\W_{ij}^{(\eps, N)} \vc = \K_\eps\left( x_i, x_j \right) 
\end{equation}
for each $i, j \in [N]$.
The \emph{degree matrix} $\D^{(\eps, N)} \in \R^{N \times N}$ is defined as
\[ \D_{ij}^{(\eps, N)} = \begin{cases}
    \sum_{k} \W_{ik}^{(\eps, N)} & i = j \\
    0 & i \neq j
\end{cases}. \]
The \emph{(unnormalized) graph Laplacian} is then given by
\begin{equation}
    \mathcal{L}^{(\eps, N)} \vc = \D^{(\eps, N)} - \W^{(\eps, N)}.
\end{equation}
We also define a closely related operator $L^{(\eps, N)}$ via
\begin{equation}
L^{(\eps, N)}f(x) \vc = \sum_{j = 1}^N \K_\eps(x, x_j)(f(x) - f(x_j)) 
\end{equation}
for any function $f \colon \M \to \R$. 
Indeed, observe that for any $f \colon \M \to \R$,
\[ \mathcal{L}^{(\eps, N)} \begin{bmatrix}
    f(x_1) \\
    \vdots \\
    f(x_N)
\end{bmatrix} = \begin{bmatrix}
        L^{(\eps, N)}f(x_1) \\
    \vdots \\
    L^{(\eps, N)}f(x_N)
\end{bmatrix}.
\]
\end{definition}
As an example, if there is a smooth embedding $\iota \colon \M \to \R^d$, then the squared Euclidean distance 
\[D(p, q) \vc = \|\iota(p) - \iota(q)\|^2\] satisfies the conditions in Assumption \ref{assump:divergence}, and $g$ in (\ref{eq:hess-riem-metric}) is the Riemannian metric on $\M$ from restricting the canonical Riemannian structure on $\R^d$ to $\iota(\M).$ We will sometimes refer to this as the Euclidean setting.
This setting is a standard assumption for manifold learning: that our data lie on a compact Riemannian submanifold of a (potentially high-dimensional) Euclidean space. In this case, one well-known dimensionality reduction algorithm is \emph{Laplacian eigenmaps} \cite{belkin2001laplacian, belkin2003laplacian,belkin2005towards},
which uses the spectral decomposition of the normalized graph Laplacian $\left(\D^{(\eps, N)}\right)^{-1}\mathcal{L}^{(\eps, N)}$\footnote{As in \cite{belkin2005towards}, but in contrast to \cite{belkin2001laplacian, belkin2003laplacian}, we do not use a $\delta$-neighborhood rule ($\W^{(\eps)}_{ij} = 0$ if $D(x_i, x_j) \geq \delta$) or a $k$NN rule ($\W^{(\eps)}_{ij} = 0$ if $x_j$ is not one of the $k$ nearest neighbors of $x_i$).}. Since $\left(\D^{(\eps, N)}\right)^{-1}\mathcal{L}^{(\eps, N)}$ is similar to the symmetric matrix $\left(\D^{(\eps, N)}\right)^{-1/2}\mathcal{L}^{(\eps, N)}\left(\D^{(\eps, N)}\right)^{-1/2}$, it has eigenvectors
\[ v^{(j)} = \left(\D^{(\eps, N)}\right)^{-1/2} w^{(j)} \]
for $j = 0, \ldots, N - 1$, where $w^{(0)}, \ldots, w^{(N-1)}$ forms an orthonormal basis of eigenvectors for $\left(\D^{(\eps, N)}\right)^{-1/2}\mathcal{L}^{(\eps, N)}\left(\D^{(\eps, N)}\right)^{-1/2}$. For each $j = 0, \ldots, N - 1$, let $\lambda_j$ denote the eigenvalue corresponding to $v^{(j)}$. Since 
\[\left(\D^{(\eps, N)}\right)^{-1}\mathcal{L}^{(\eps, N)} = I - \left(\D^{(\eps, N)}\right)^{-1}\mathcal{W}^{(\eps, N)}, \] 
without loss of generality, we can order the eigenvalues
\[ 0 = \lambda_0 \leq \lambda_1 \leq \cdots \leq \lambda_{N-1} \]
and take
$v^{(0)}$ to be a multiple of the all ones vector. For a fixed $l \leq N-1$, Laplacian eigenmaps maps each data point $x_i$ into $\R^l$ using the eigenvectors $v^{(1)}, \ldots, v^{(l)}$:
\begin{equation} 
\label{eq:lapeigenmap}
x_i \mapsto \begin{bmatrix}
    v^{(1)}_i \\
    \vdots \\
    v^{(l)}_i
\end{bmatrix}.
\end{equation}
Even when we are only given $h$ as in Assumption \ref{assump:lap-eig-kernel} and $D$ as in Assumption \ref{assump:divergence} (not necessarily the squared Euclidean distance), we can still apply the Laplacian eigenmap algorithm, in the sense that we can construct the matrix $\left(\D^{(\eps, N)}\right)^{-1}\mathcal{L}^{(\eps, N)}$ and compute its eigenvectors.

\subsection{Sinkhorn divergences}
\label{subsec:background-sink-div}
We now recall some facts about entropy-regularized optimal transport, which can be found in \cite{peyre2019computational, nutz2022introduction, chewi2024statisticaloptimaltransport}.
Let $X \subset \R^d$ be compact. 
For $\mu, \nu \in \mathcal{P}(X)$, define for each $\beta > 0$
\begin{equation}
\label{eq:ent-reg-ot}
    \mathrm{OT}_\beta(\mu, \nu) \vc= \inf_{\pi \in \Pi(\mu, \nu)} \int \|x-y\|^2 d\pi(x, y) + \beta \mathrm{KL}(\pi | \mu \otimes \nu),
\end{equation}
where $\Pi(\mu, \nu)$ denotes the set of couplings between $\mu$ and $\nu$ and $\mathrm{KL}$ denotes the Kullback–Leibler divergence. (For this paper, we will only consider the squared Euclidean distance as the cost function.)
Then for each $\mu, \nu \in \mathcal{P}(X)$, as $\beta \to 0$, $\mathrm{OT}_\beta(\mu, \nu)$ converges to the square of the Wasserstein-2 distance
\[ W_2(\mu,\nu) \vc = \left( \inf_{\pi \in \Pi(\mu, \nu)} \int \|x-y\|^2 d\pi(x, y)\right)^{1/2} \]
between $\mu$ and $\nu$ \cite[Section 5]{nutz2022introduction}.
However, for $\beta > 0$ and $\mu \in \mathcal{P}(X)$, $\mathrm{OT}_\beta(\mu, \mu)$ may be non-zero, so
 \cite{genevay2018learning} proposed to instead use the Sinkhorn divergence
\begin{equation}
\label{eq:def-sinkdiv}
    S_\beta(\mu, \nu) \vc = \mathrm{OT}_\beta(\mu, \nu) - \frac{\mathrm{OT}_\beta(\mu, \mu)}{2} - \frac{\mathrm{OT}_\beta(\nu, \nu)}{2}
\end{equation}
as a loss function for machine learning.
In \cite{feydy2019interpolating}, it is shown that for each $\beta > 0,$ $S_\beta$ is non-negative and $S_\beta(\mu, \nu) = 0$ if and only if $\mu = \nu$. 

From now on, we assume that $\beta > 0$ is fixed.
The dual for (\ref{eq:ent-reg-ot}) is given by
\begin{equation}
    \label{eq:ent-reg-ot-dual}
    \mathrm{OT}_\beta(\mu, \nu) = \sup_{f, g \in C(X)} \int f d\mu + \int g d\nu - \beta \int \left( \exp\left( \frac{f(x) + g(y) - \|x-y\|^2}{\beta}\right)-1 \right)d(\mu \otimes \nu)(x, y)
\end{equation}
for $\mu, \nu \in \mathcal{P}(X),$ and there exist $f_{\mu, \nu}, g_{\mu, \nu} \in C(X)$, also known as Schr\"odinger potentials, which attain the maximum in (\ref{eq:ent-reg-ot-dual}). 
In addition, we have that $\pi_{\mu, \nu}$ given by
\begin{equation}
\label{eq:optimal-coupling-form}
    d\pi_{\mu, \nu}(x, y) = \exp \left( \frac{f_{\mu, \nu}(x) + g_{\mu,\nu}(y) - \|x-y\|^2}{\beta} \right)d(\mu \otimes \nu)(x, y)
\end{equation}
is optimal for the primal (\ref{eq:ent-reg-ot}).
Since $\pi_{\mu, \nu}$ is a coupling between $\mu$ and $\nu$, $f_{\mu, \nu}, g_{\mu, \nu}$ satisfy
\begin{equation}
\label{eq:schrodinger-1}
    \int \exp \left(\frac{f_{\mu, \nu}(x) + g_{\mu, \nu}(y) - \|x-y\|^2}{\beta} \right) d\nu(y) = 1 \quad \text{for $\mu$-a.e. $x$}
\end{equation}
\begin{equation}
\label{eq:schrodinger-2}
    \int \exp \left(\frac{f_{\mu, \nu}(x) + g_{\mu, \nu}(y) - \|x-y\|^2}{\beta} \right) d\mu(x) = 1 \quad \text{for $\nu$-a.e. $y$}
\end{equation}
which can also be understood as the Euler-Lagrange equations or first-order optimality conditions for (\ref{eq:ent-reg-ot-dual}) \cite[Remark 3.4]{nutz2022introduction}. 

Observe that maximizers of (\ref{eq:ent-reg-ot-dual}) are not unique, since if $f_{\mu, \nu}, g_{\mu, \nu}$ is a maximizer, so is $f_{\mu, \nu} +\lambda, g_{\mu, \nu} - \lambda$ for any $\lambda \in \R$. 
Nonetheless, the Schr\"odinger potentials $f_{\mu, \nu}, g_{\mu, \nu}$ are unique up to an additive constant upon requiring that (\ref{eq:schrodinger-1}) and (\ref{eq:schrodinger-2}) hold for all $x, y \in X$ (rather than $\mu$-a.e. $x$ and $\nu$-a.e. $y$), i.e.,
\begin{equation}
\label{eq:schrodinger-1-1}
f_{\mu, \nu}(x) = -\beta \log \left( \int \exp \left( \frac{g_{\mu, \nu}(y) - \|x-y\|^2}{\beta} \right) d\nu(y) \right) 
\end{equation}
\begin{equation}
\label{eq:schrodinger-2-1}
g_{\mu, \nu}(y) = -\beta \log \left( \int \exp \left( \frac{f_{\mu, \nu}(x) - \|x-y\|^2}{\beta} \right) d\mu(x) \right) 
\end{equation}
for all $x, y \in X$.

\subsection{Related work}
\label{subsec:related-work}
The Laplacian eigenmap \cite{belkin2001laplacian, belkin2003laplacian,belkin2005towards} and diffusion map \cite{coifman2005geometric, coifman2006diffusion, nadler2006diffusion} algorithms can be viewed as discrete analogues to embeddings of $(\M, g)$ using the eigenfunctions of the Laplace-Beltrami operator \cite{bates2014embedding} and heat kernel \cite{berard1994embedding}, respectively. As such, under the setting where $(\M, g)$ is a compact Riemannian submanifold of $\R^d$ and $D$ is the squared Euclidean distance, several notions of the convergence of graph Laplacians to the Laplace-Beltrami operator (or a weighted version) on $(\M, g)$ have been studied; see, for example, \cite{belkin2005towards, belkin2008towards, hein2005graphs, singer2006graph, coifman2006diffusion, belkin2006convergence, garcia2020error, cheng2022eigen, calder2022improved, calder2022lipschitz, dunson2021spectral, wahl2024kernel, trillos2025minimax} and references therein. 

However, in several manifold learning algorithms (e.g., Isomap \cite{tenenbaum2000global} in addition to the Laplacian eigenmap and diffusion map algorithms), one can replace the Euclidean distance with another measure of dissimilarity between data points. Indeed, depending on the application, other measures of dissimilarity may be more appropriate than the Euclidean distance; for examples, see \cite{carter2009fine, lu2023discovering, kileel2021manifold, zelesko2020earthmover, tenenbaum2000global, mishne2016hierarchical}. 
The particular application which motivates our Assumption \ref{assump:manifold-hyp} is that of understanding the conformation space of a molecule \cite{diepeveen2024riemannian, zelesko2020earthmover}. When comparing two conformations of a molecule, an optimal transport-related distance is better for measuring displacements in space, 
especially when compared to the $L^2$ distance between the densities. Therefore, it is not surprising that Zelesko et al. found using a wavelet approximation of the Wasserstein-1 metric \cite{shirdhonkar2008approximate} in the diffusion maps algorithm beneficial when analyzing the conformation space of a molecule \cite{zelesko2020earthmover}.

Motivated by \cite{zelesko2020earthmover}, \cite{kileel2021manifold} studied the pointwise convergence of the graph Laplacian when the typical Euclidean norm is replaced by an arbitrary norm on $\R^d$, using tools from convex geometry. In \cite{trillos2024fermat}, the spectral convergence of graph Laplacians constructed with Fermat distances is studied, though we note that Fermat distances depend on the data and thus our analysis will not apply. In \cite{oliver2025laplace}, Oliver et al. study the consistency of Laplace learning on a submanifold of Wasserstein-2 space (see \cite{oliver2025laplace} and \cite{hamm2024manifoldlearningwassersteinspace} for the precise assumptions) via $\Gamma$-convergence of the discrete Dirichlet energies.
Though we focus on graph Laplacian-based methods, several authors have also studied adapting other dimensionality reduction and/or manifold learning methods to a space of probability measures; see, e.g., \cite{hamm2023wassmap, bigot2017geodesic, seguy2015principal, hamm2024manifoldlearningwassersteinspace, arias2025embedding, warren2025principal}. 

The appeal of the Wasserstein-2 space is that it has a formal Riemannian structure, introduced by Otto in \cite{otto2001geometry}. For example, the Benamou-Brenier dynamical formulation of the Wasserstein-2 distance \cite{benamou2000computational} is similar to how length-minimizing geodesics can be characterized as energy minimizers in Riemannian geometry. We refer to \cite{villani2009optimaloldandnew, ambrosio2008gradient, chewi2024statisticaloptimaltransport} for more details on this Riemannian structure.
However, the Wasserstein-2 space is not a true finite-dimensional smooth Riemannian manifold, so much care needs to be taken even when constructing a "Riemannian submanifold" of the Wasserstein-2 space \cite{hamm2024manifoldlearningwassersteinspace} and defining the Laplace-Beltrami operator on such a submanifold \cite{oliver2025laplace}. Additional regularity constraints can bring us a bit closer to standard Riemannian geometry \cite{lott2008some}, but the space that \cite{lott2008some} works with is still infinite-dimensional.

In this work, we instead adapt our approach from \cite{xu2025manifold}, so that we only need to work with a finite-dimensional smooth Riemannian manifold.
In \cite{xu2025manifold}, we observed that when a smooth compact orientable manifold $\M$ is embedded into a metric space $(X, d)$, a sufficient condition for the pointwise convergence of graph Laplacians constructed with $d^2$ is for $d^2$ to be smooth on a neighborhood of the diagonal $\Xi_\M$ of $\M \times \M$ and for $\frac{1}{2}\mathrm{Hess}_p d^2(p, \cdot)$ to be nondegenerate for all $p \in \M$.
The main idea is to equip $\M$ with the Riemannian metric 
\begin{equation}
\label{eq:induced-riem-metric}
g_p \vc = \frac{1}{2}\mathrm{Hess}_p d^2(p, \cdot)
\end{equation}
and establish that there exists $K, \eps_0 > 0$ such that
\[
d(p,q)^2 \leq d_g(p, q)^2 \leq d(p, q)^2 + K d_g(p, q)^4
\]
for all $p, q\in \M$ such that $d_g(p, q) < \eps_0$. From here, the pointwise convergence of the graph Laplacians follows \textit{mutatis mutandis} from the Euclidean setting \cite{belkin2005towards}; see \cref{sec:gl-conv} and Appendix \ref{sec:gl-conv-app} for more details.

More generally, we can consider
\begin{equation}
\label{eq:induced-riem-metric-div}
g_p \vc = \frac{1}{2}\mathrm{Hess}_p D(p, \cdot)
\end{equation}
as a Riemannian metric
so long as $D \colon \M \times \M \to \R_{\geq 0}$ is smooth on a neighborhood of $\Xi_\M$, $D(p, p) = 0$ for all $p \in \M$ and (\ref{eq:induced-riem-metric-div}) is nondegenerate for all $p \in \M$.
This idea dates back to Rao \cite{rao1987differential} and is central to information geometry \cite{amari2016information, ay2017information}, which uses the Fisher information as a Riemannian metric. 
Namely, when $D$ is given by, e.g., the Kullback–Leibler divergence or squared Hellinger distance, under some assumptions, one can show that (\ref{eq:induced-riem-metric-div}) is a scalar multiple of the Fisher information \cite{chewi2024statisticaloptimaltransport, kullback1997information}. 
More concretely, in statistics, one often considers a family of probability measures $\{\mu_x \colon x \in \Theta\}$ parametrized by a subset $\Theta \subset \R^m$; under appropriate assumptions, we can define $D(x, y) \vc = \mathrm{KL}(\mu_x || \mu_y)$, upon which (\ref{eq:induced-riem-metric-div}) is half of the Fisher information at $x$ \cite{ kullback1997information}. (Note, however, that in Assumption \ref{assump:divergence}, we also require $D$ to be symmetric.) This can then be used in the natural gradient algorithm \cite{amari1998natural}, a gradient descent algorithm where the gradient is taken with respect to $g$. Combining information geometry and the Riemannian structure of Wasserstein-2 space, more recent work has studied when the Kullback-Leibler divergence is replaced with (a variant of) the squared Wasserstein-2 distance \cite{chen2020optimal, li2023wasserstein, li2018natural}. (\ref{eq:induced-riem-metric}) and/or (\ref{eq:induced-riem-metric-div}) have also been used to study Isomap on a family of measures parametrized by an open set of $\R^m$ \cite{arias2025embedding} and to approximate geodesics on a Riemannian manifold \cite{diepeveen2024riemannian, rumpf2015variational}.

Given the prevalence of the Sinkhorn divergence in machine learning and related areas, a natural direction then is to study (\ref{eq:induced-riem-metric-div}) on a parametrized family of probability measures when $D$ is given by the Sinkhorn divergence. This is done in \cite{shen2020sinkhorn}, which is more focused on the natural gradient algorithm, and \cite{Lavenant2025}, which explores the properties of the Riemannian structure given by (\ref{eq:induced-riem-metric-div}) in more detail. For this paper, we will follow \cite{Lavenant2025}.
We also mention the follow-up work \cite{hardion2025gradient}, which uses \cite{Lavenant2025} to understand a variant of the JKO scheme with the Sinkhorn divergence. To the best of our knowledge, there has not been prior work studying graph Laplacian-based manifold learning algorithms when the Riemannian structure is given by the Hessian of the Sinkhorn divergence.

\section{Approximations of the geodesic distance and convergence of the graph Laplacian}
\label{sec:gl-conv}
It is well-known that for a smooth, connected, compact Riemannian submanifold $(\M, g)$ of $\R^d$, there exists $K > 0$ such that
\begin{equation} 
\label{eq:4th-order-error}
\|p-q\|^2 \leq d_g(p, q)^2 \leq \|p-q\|^2 + Kd_g(p, q)^4
\end{equation}
for all $p, q \in \M$, where $d_g$ denotes the geodesic distance on $(\M, g)$; see, e.g., \cite[Lemma 4.3]{belkin2008towards} or \cite[Lemma 3]{bernstein2000graph}. 
Indeed, (\ref{eq:4th-order-error}) is often helpful when analyzing manifold learning algorithms for data lying on $\M \subset \R^d$.

Now consider the more general case, where we are only given a manifold $\M$ satisfying the assumptions in Table \ref{tab:notation} and a function $D$ satisfying Assumption \ref{assump:divergence}. Endow $\M$ with the Riemannian metric $g$ as defined in (\ref{eq:hess-riem-metric}). For the rest of this section and Appendix \ref{sec:error-proof-app}, let $\exp$ and $d_g$ denote the associated exponential map and geodesic distance, respectively. Fix any $p \in \M$ and $v \in T_p \M$ such that $g_p(v, v) = 1$. 
A priori, from Taylor expanding around $p$, we may expect $D$ to approximate the geodesic distance $d_g$ up to third order, i.e., for $t \in \R$ with $|t|$ small enough,
\begin{equation}
\label{eq:tay-exp-third-order}
D(p,\exp_p(tv)) = t^2 + O(|t|^3) = d_g(p,\exp_p(tv))^2 + O(|t|^3)
\end{equation}
as $t \to 0$.
If $D = d^2$ for a metric $d$ on $\M$, then one can obtain that there exists $K > 0$ such that
\begin{equation}
\label{eq:squared-metric-geo-dist-comparison}
d(p,q)^2 \leq d_g(p, q)^2 \leq d(p, q)^2 + K d_g(p, q)^4 
\end{equation}
for all $p, q \in \M$ using the triangle inequality and compactness of $\M$ \cite{xu2025manifold}. However, a divergence is not necessarily the square of a distance, and in particular, the Sinkhorn divergence is not the square of a distance \cite{Lavenant2025}. Instead, we will use the symmetry of $D$ in Assumption \ref{assump:divergence} to improve upon the third-order error in (\ref{eq:tay-exp-third-order}):

\begin{proposition}
\label{prop:fourth-order-error}
    With $D$ satisfying Assumption \ref{assump:divergence}, there exists $K > 0$ such that
    \[ \left| D(p, q) - d_g(p, q)^2 \right| \leq K d_g(p, q)^4 \]
    for all $p, q \in \M$.
\end{proposition}
The main idea is to apply the following lemma to $D$ along the unit-speed geodesics of $(\M, g)$:
\begin{lemma}
\label{lemma:third-order-vanish}
Suppose $f \in C^4([a, b] \times [a, b])$ for some $a < b$ and that $f$ is symmetric, i.e., \[f(x, y) = f(y, x)\] for all $x, y \in [a, b]$, that $f \geq 0$ and $f(x, x) = 0$ for all $x \in [a, b]$. 
Let $f_x \vc = f(x, \cdot)$ for each $x \in [a, b]$. 
Then
\begin{equation}
\label{eq:3rd-order-derivs}
\frac{\partial ^3 f}{\partial y^3}(x, x) = -3 \frac{\partial^3 f}{\partial x \partial y^2}(x, x) 
\end{equation}
for all $x \in (a, b)$.
In particular, if additionally we have $f_x''(x) = 2$ for all $x \in (a, b)$, then $f_x'''(x) = 0$ for all $x \in (a, b)$ and so
\begin{equation}
\label{eq:4th-order-error-2}
\left| f(x, y) - (x-y)^2 \right| \leq \frac{1}{24} \left( \sup_{z \in (a, b)} \left|f_x^{(4)}(z) \right| \right) (x-y)^4 \end{equation}
for all $x, y \in (a, b)$.
\end{lemma}
\cref{lemma:third-order-vanish} follows from Taylor expanding each $f_x$ and using the symmetry condition; see Appendix \ref{sec:error-proof-app}.
We also leave the proof of \cref{prop:fourth-order-error} for Appendix \ref{sec:error-proof-app}, as it is similar to \cite[Lemma 3.5]{xu2025manifold}.
As we cannot use the triangle inequality, in contrast to (\ref{eq:squared-metric-geo-dist-comparison}), we may not have $D(p, q) \leq d_g(p, q)^2$ for all $p, q \in \M$. Nonetheless, \cref{prop:fourth-order-error} is enough for us to prove, for example, the pointwise convergence of the graph Laplacian, following similar arguments as the Euclidean setting \cite{belkin2008towards, coifman2006diffusion}.

\begin{theorem}[c.f. \cite{belkin2008towards, coifman2006diffusion}]
\label{thm:glconv}
Assume Assumptions \ref{assump:lap-eig-kernel} and \ref{assump:divergence}. Fix any $f \in C^3(\M)$, $x \in \M$ and $\alpha > 0$, and suppose $\eps_N = \Omega\left(N^{-\frac{1}{m+2 + \alpha}}\right)$ and $\eps_N \to 0$ as $N \to \infty$.
Define (as in \cite{coifman2006diffusion})
\begin{equation}
\label{eq:def_m_2}
m_2 \vc = \int_{\R^m} v_1^2 h(\|v\|^2) dv. 
\end{equation}
If $x_1, \ldots, x_N$ are drawn i.i.d. from a probability distribution on $\M$ with a density $P \in C^3(\M)$ with respect to the Riemannian volume form $dV_g$ (where $g$ is as given in Assumption \ref{assump:divergence}), then 
    \[\frac{L^{(\eps_N, N)} f(x)}{N\eps_N^{m/2 + 1}}  \to \frac{m_2}{2} \left( P(x) \Delta_g f(x) - 2g_x( \grad_g f(x), \grad_g P (x) ) \right) \]
    almost surely as $N \to \infty$.
    In particular, if $x_1, \ldots, x_N$ are drawn i.i.d. from the uniform distribution on $(\M, g)$, then
    \[\frac{L^{(\eps_N, N)} f(x)}{N\eps_N^{m/2 + 1}} \to \frac{m_2}{2} \left(\frac{\Delta_g f(x)}{\mathrm{vol}_g(\M)} \right) \]
almost surely as $N \to \infty$.
\end{theorem}

As it is very similar to the arguments in \cite{belkin2008towards, coifman2006diffusion}, we leave the proof of \cref{thm:glconv} to Appendix \ref{sec:gl-conv-app}. We expect other results, such as the spectral convergence of graph Laplacians (see \cite{calder2022improved, cheng2022eigen} and references therein for the Euclidean setting), to hold as well, but we have not checked this for this work.

\section{Application to Sinkhorn divergences}
\label{sec:sink-div}
In this section, we discuss two models where a family of measures is parametrized by a manifold $\M$ and the Sinkhorn divergence is smooth around the diagonal of $\M \times \M$. We then discuss when the non-degeneracy condition (5. in Assumption \ref{assump:divergence}) holds. Throughout this section, $\beta > 0$ is fixed.

\subsection{Regularity}

We start with showing that the smoothness around the diagonal assumption (4. in Assumption \ref{assump:divergence}) holds in two models of interest.

\subsubsection{Particle system parametrized by a manifold}
We first consider a particle system parametrized by a manifold:

\manifoldhyp

We claim that
$
    D(x, y) \vc = S_\beta(F(x), F(y))
$
is smooth on $\M \times \M$ under Assumption \ref{assump:manifold-hyp}. This follows immediately from the following lemma.

\begin{lemma}
\label{lemma:ot-beta-smooth}
Under Assumption \ref{assump:manifold-hyp},
\[ G_\beta(x, y) \vc = \mathrm{OT}_\beta(F(x), F(y)) \]
is smooth on $\M \times \M$.
\end{lemma}

The main idea is to use the implicit function theorem with the first-order optimality conditions (\ref{eq:schrodinger-1-1}) and (\ref{eq:schrodinger-2-1}).
The proof is similar to the proof of \cite[Theorem 2]{luise2018differential}, but in contrast to \cref{lemma:ot-beta-smooth}, \cite{luise2018differential} assumes that the locations of the Dirac delta functions are fixed, whereas the coefficients $p_j$ are allowed to vary.

The implicit function theorem approach has also been taken in \cite{Lavenant2025, carlier2024displacement}. However, in these works, the measures are not necessarily assumed to be discrete with finite support, and so the proofs become more technical. 
We will discuss this more in \cref{subsec:smooth-deform}, but here we first consider the particle system assumption (Assumption \ref{assump:manifold-hyp}) to illustrate the main ideas.

\begin{proof}[Proof of \cref{lemma:ot-beta-smooth}]
For convenience, define $\mathcal{F}_\beta \colon \R^n \times \R^n \times \M \times \M \to \R$ as
\begin{equation}
    \label{eq:lagrangian}
    \mathcal{F}_\beta (f, g, x, y) \vc = \sum_{j = 1}^n (f_j + g_j) p_j - \beta \sum_{j, k} \left(\exp \left( \frac{f_j + g_k - \|\iota_j (x) - \iota_k(y)\|^2}{\beta} \right)-1\right)p_jp_k,
\end{equation}
so that by (\ref{eq:ent-reg-ot-dual}), we have
\begin{equation}
\label{eq:g-beta}
G_\beta(x, y)  = \sup_{f, g \in \R^n} \mathcal{F}_\beta(f, g, x, y). 
\end{equation}
As discussed in \cref{subsec:background-sink-div}, this problem admits maximizers (the Schr\"odinger potentials), though we also have
\[ \mathcal{F}_\beta(f + \lambda \ind, g- \lambda \ind, x, y) = \mathcal{F}_\beta(f, g, x, y) \]
for all $\lambda \in \R$ and $x, y \in \M$, where $\ind$ denotes the $n$-dimensional vector of all ones. However, if we consider $\R^n \times \R^n / \sim$, where $(f, g) \sim (\widetilde{f}, \widetilde{g})$ if $\widetilde{f} = f + \lambda \ind, \widetilde{g} = g - \lambda \ind$ for some $\lambda \in \R$, 
then for each $x, y \in \M$, $\mathcal{F}_\beta(\cdot, \cdot, x, y)$ is strictly concave on $\R^n \times \R^n / \sim$.

For concreteness, let us identify $\R^n \times \R^n / \sim$ with the subspace $P = \{(f, g) \in\R^n \times \R^n | f_1 = 0\}$, in which case the strict concavity of $\mathcal{F}_\beta (\cdot, \cdot, x, y)$ on $P$ implies that the maximizer $f^*(x, y), g^*(x, y)$ of $\sup_{f, g \in P} \mathcal{F}_\beta(f, g, x, y)$ is unique in $P$ for each $x, y \in \M$. 
Define
\begin{equation}
\label{eq:t-beta-discrete}
    \mathcal{T}^{(\beta)} (f, g, x, y) \vc = \nabla_{f_2, \ldots, f_n, g_1, \ldots, g_n} \mathcal{F}_\beta (f, g, x, y), 
    \end{equation}
where $\nabla_{f_2, \ldots, f_n, g_1, \ldots, g_n}$ denotes the gradient in the $f_2, \ldots, f_n, g_1, \ldots, g_n$ directions.
For each $x, y \in \M$, using the first-order optimality conditions for (\ref{eq:g-beta}) and the strict concavity of $\mathcal{F}_\beta (\cdot, \cdot, x, y)$, $f^*(x, y), g^*(x, y) \in P$ are characterized by
\begin{equation}
\label{eq:first-order-optimality}
\mathcal{T}^{(\beta)}(f^*(x, y), g^*(x, y), x, y) = 0.
\end{equation}
Observe that we can write
the first-order optimality conditions for (\ref{eq:g-beta}) more explicitly as
\begin{align*}
    p_j - p_j \sum_k \exp \left( \frac{f_j + g_k - \| \iota_j(x) - \iota_k(y)\|^2}{\beta} \right)p_k &= 0 \\
    p_j - p_j \sum_k \exp \left( \frac{f_k + g_j - \| \iota_k(x) - \iota_j(y)\|^2}{\beta} \right)p_k &= 0,
\end{align*}
for each $j = 1, \ldots, n$, 
which, upon rearranging, are precisely the Schr\"odinger system (\ref{eq:schrodinger-1}) and (\ref{eq:schrodinger-2}).
For each $x, y \in \M$, when restricted to a bounded convex subset of $P$, $\mathcal{F}_\beta(\cdot, \cdot, x, y)$ is strongly concave with respect to the norm
\[ \vert\vert\vert (f, g) \vert\vert\vert \vc = \left( \sum_{j, k = 1}^n \|f_j + g_k\|^2 \right)^{1/2} ,\]
hence with respect to the Euclidean norm as well. (Here $\vert\vert\vert \cdot \vert\vert\vert$ is only a norm on $P$ — not on $\R^n \times \R^n$ — since we require $f_1 = 0$ for all $(f, g)  \in P$.)
Therefore for each $x, y \in \M$, the Jacobian of $\mathcal{T}^{(\beta)}$ at $f^*(x, y), g^*(x, y), x, y$ with respect to the $f_2, \ldots, f_n, g_1, \ldots, g_n$ variables
is invertible, so the implicit function theorem implies that $f^*$ and $g^*$ are smooth functions on a neighborhood of $(x, y)$. Hence $G_\beta$ is smooth on a neighborhood of $(x, y)$ for any $x, y \in \M$.
\end{proof}

\subsubsection{Smooth deformations of a measure}
\label{subsec:smooth-deform}
In this subsection, we consider measures parametrized by an open set $U \subset \R^m$ instead of $\M$, but our results can be extended to a manifold by passing to charts.

\manifoldhypdeform

\begin{proposition}
\label{prop:smoothness-of-OT-beta}
    Under Assumption \ref{assump:manifold-hyp-deform},
    \[ G_\beta(p, q) \vc = \mathrm{OT}_\beta(F(p), F(q)) \]
    is $C^k$ on $U\times U$. 

\end{proposition}
\begin{remark}
    In Assumption \ref{assump:divergence}, we assumed that $D$ is smooth (i.e., infinitely differentiable) on a neighborhood of the diagonal $\Xi_\M \subset \M \times \M$, but with a careful reading one can track the precise number of derivatives required for \cref{prop:fourth-order-error} and \cref{thm:glconv}. Alternatively, one can assume that Assumption \ref{assump:manifold-hyp-deform} holds for all $k \in \Z_{> 0}$, upon which \cref{prop:smoothness-of-OT-beta} implies that $D(p, q) \vc = S_\beta(F(p), F(q))$ is infinitely differentiable on $U \times U$.
\end{remark}
We will prove \cref{prop:smoothness-of-OT-beta} at the end of this subsection, after first establishing a few lemmata.
Define (as in \cite[Definition 2.1]{carlier2024displacement}) the map $T^{(\beta)} \colon C(X)^2 \times \mathcal{P}(X) ^2 \to C(X)^2$ by
\begin{equation}
\label{eq:t-beta-def-1}
T^{(\beta)}_1(f, g, \mu, \nu)(x)\vc = \log \left( \int \exp \left(\frac{f(x) + g(y) - \|x-y\|^2}{\beta} \right) d\nu(y) \right)
\end{equation}
\begin{equation}
\label{eq:t-beta-def-2}
T^{(\beta)}_2(f, g, \mu, \nu)(x)\vc = \log \left( \int \exp \left(\frac{f(y) + g(x) - \|x-y\|^2}{\beta} \right) d\mu(y) \right)
\end{equation}
for $f, g \in C(X)$ and $\mu, \nu \in \mathcal{P}(X)$. 

\begin{lemma}
\label{lemma:T_beta-smooth}
    Under Assumption \ref{assump:manifold-hyp-deform},
    the map 
    \[(f, g, p, q) \mapsto T^{(\beta)}(f, g, F(p), F(q)) \]
    is a $C^k$ map from $C^k(X) \times C^k(X) \times U \times U$ to $C^k(X) \times C^k(X)$.

\end{lemma}
\begin{remark}
\label{rmk:c-k-whitney-extension}
    Here we use $C^k(X)$ to refer to the space of functions in $C^k(\Omega)$ whose derivatives of order $\leq k$ have continuous extensions to $X = \overline{\Omega}$, equipped with the norm
    \[ \|f\|_{C^k(X)} \vc = \sup_{|\alpha| \leq k, x \in \Omega} \left|\partial^\alpha f(x)\right|, \]
    which makes it a Banach space. If one only assumes that $X$ is compact and convex, then \cite{carlier2024displacement} (which we will use in the proof of \cref{lemma:schro-potentials-regularity}) instead defines $C^k(X)$ as the space of functions on $X$ with a $C^k$ extension on $\R^d$ with the norm 
    \[ |||f|||_{C^k(X)} \vc =\inf_{\widetilde{f}}\sup_{|\alpha| \leq k, x\in \R^d} \left|\partial^\alpha\widetilde{f}(x)\right|,\] where the infimum is taken over all $C^k$ extensions of $f$ to $\R^d$.
    In our setting, where $X = \overline{\Omega}$ with $\Omega \subset \R^d$ open, bounded and convex, these two definitions are equivalent (up to equivalence of norms) by Whitney's extension theorem \cite{whitney1934functions} (see also \cite[Theorem 2.3.6 and proof of Theorem 2.3.10]{hormander-I}).
    
In particular, since any $f \in C^k(X)$ has a $C^k$ extension $\widetilde{f}$ on an open neighborhood of $X$, under Assumption \ref{assump:manifold-hyp-deform} we can apply the chain rule to $p \mapsto f(\Psi(p, x))$ for each $x \in X$ and $f \in C^k(X)$.
\end{remark}
We leave the proof of \cref{lemma:T_beta-smooth} for Appendix \ref{sec:sink-div-app}, though we note that
this lemma is similar to \cite[Lemma 3.4]{carlier2024displacement} and \cite[Lemma 3.13]{Lavenant2025}. 
\cite[Lemma 3.4]{carlier2024displacement} is primarily concerned with families of measures arising from interpolations, whereas
\cite[Lemma 3.13]{Lavenant2025} assumes different conditions, which are satisfied, for instance, by $\mu_t = \Psi(t, \cdot)_{\#}\mu$ for $t \in (-\delta, \delta)$, for some $\delta > 0$ and $\Psi$ with only one continuous derivative with respect to the first variable (what they call \textit{horizontal perturbations}) \cite[Remark 3.2]{Lavenant2025}. Therefore \cite[Lemma 3.13]{Lavenant2025} only proves the continuous differentiability of $T^{(\beta)}$. We assume more regularity of $\Psi$ with respect to the first variable (when $k > 1$) to obtain additional regularity of $T^{(\beta)}$.

Following the approach in \cref{lemma:ot-beta-smooth}, we now apply the implicit function theorem.

\begin{lemma}
\label{lemma:schro-potentials-regularity}
    Assume Assumption \ref{assump:manifold-hyp-deform}.
    For each $p, q \in U$, let $f_{p, q}$ and $g_{p, q}$ denote the Schr\"odinger potentials between $F(p) = \Psi(p, \cdot)_{\#}\mu$ and $F(q) = \Psi(q, \cdot)_{\#}\mu$. 
    Then $p, q \mapsto f_{p, q}, g_{p, q}$ is a $C^{k}$ function from $U \times U$ to $\widetilde{C}^k(X) \vc = C^k(X) \times C^k(X)/\sim$, where 
    \[(f, g) \sim (\widetilde{f}, \widetilde{g}) \Leftrightarrow \widetilde{f} = f + \lambda, \widetilde{g} = g - \lambda \text{ for some }\lambda \in \R.\]
    Here we equip $\widetilde{C}^k(X)$ with the quotient norm, so that it is a Banach space.
\end{lemma}

\begin{proof}
For each $p, q \in U$, the Schr\"odinger system (\ref{eq:schrodinger-1-1}) and (\ref{eq:schrodinger-2-1}) for $F(p) = \Psi(p, \cdot)_{\#}\mu$ and $F(q) = \Psi(q, \cdot)_{\#}\mu$ can be rewritten as
\begin{equation}
    \label{eq:t-beta-zero}
    T^{(\beta)}(f_{p, q}, g_{p, q}, F(p), F(q)) = 0.
\end{equation}
Using the definition of $T^{(\beta)}$ and \cref{lemma:T_beta-smooth}, 
 $(f, g, p, q) \mapsto T^{(\beta)}(f, g, F(p), F(q))$ is well-defined as a map 
from $\widetilde{C}^k(X) \times U \times U$ to $\widetilde{C}^k(X)$, and moreover it is $C^k$.
In addition, 
by \cite[Lemma 3.2]{carlier2024displacement} (see also \cite{gonzalez2022weak}), for any $\mu, \nu \in\mathcal{P}(X)$,
$D_{f, g} T^{(\beta)}(f_{\mu, \nu}, g_{\mu,\nu}, \mu, \nu)$ is a Banach space isomorphism from $\widetilde{C}^k(X)$ to itself. Here $D_{f, g}$ denotes the derivative with respect to the $\widetilde{C}^k(X)$ variable. 
Although we stated (\ref{eq:t-beta-zero}) as an equivalence of functions in $C^k(X) \times C^k(X)$, if for some $f, g \in C^k(X)$, $\mu, \nu \in \mathcal{P}(X)$ and $\lambda \in \R$
\[\
T^{(\beta)}_1(f, g, \mu, \nu)(x) = \log \left( \int \exp \left(\frac{f(x) + g(y) - \|x-y\|^2}{\beta} \right) d\nu(y) \right) = \lambda
\]
\[
T^{(\beta)}_2(f, g, \mu, \nu)(x) = \log \left( \int \exp \left(\frac{f(y) + g(x) - \|x-y\|^2}{\beta} \right) d\mu(y) \right) = -\lambda
\]
for all $x \in X$, then
\[ e^\lambda = \int \int \exp \left(\frac{f(x) + g(y) - \|x-y\|^2}{\beta} \right) d\nu(y)d\mu(x) = e^{-\lambda}, \]
so $\lambda$ must be 0. Therefore, even if we take (\ref{eq:t-beta-zero}) as an equivalence of functions in $\widetilde{C}^k(X)$, it characterizes the Schr\"odinger potentials for $F(p) = \Psi(p, \cdot)_{\#}\mu$ and $F(q) = \Psi(q, \cdot)_{\#}\mu$.
The lemma statement then follows from the implicit function theorem for Banach spaces \cite[Theorem 5.9]{lang1999diffgeo}.

\end{proof}

By \cref{lemma:schro-potentials-regularity} and the regularity of $\Psi$ in Assumption \ref{assump:manifold-hyp-deform}, we obtain the following:

\begin{lemma}
\label{lemma:regularity-of-integrand}
    Following the notation from \cref{lemma:schro-potentials-regularity}, under Assumption \ref{assump:manifold-hyp-deform}, 
    \[ p, q \mapsto f_{p, q}(\Psi(p, \cdot)), g_{p,q}(\Psi(q, \cdot))\]
    is a $C^k$ map
    from $U \times U$ to $\widetilde{C}(X) \vc = C(X) \times C(X) /\sim$. As with $\widetilde{C}^k(X)$, the space $\widetilde{C}(X)$ is equipped with the quotient norm, making it a Banach space.
\end{lemma}
As with \cref{lemma:T_beta-smooth}, the proof involves checking that partial derivatives exist and are continuous with respect to appropriate norms; we leave the proof for Appendix \ref{sec:sink-div-app}.
Using \cref{lemma:regularity-of-integrand}, we can now prove \cref{prop:smoothness-of-OT-beta}.

\begin{proof}[Proof of \cref{prop:smoothness-of-OT-beta}]
    Under Assumption \ref{assump:manifold-hyp-deform}, we have
    \[ G_\beta(p, q) = \mathrm{OT}_\beta(F(p), F(q)) = \int f_{p, q} d\mu_p + \int g_{p, q} d\mu_q = \int f_{p, q}(\Psi(p, x)) + g_{p, q}(\Psi(q, x)) d\mu(x). \]
    Although $f_{p, q}, g_{p,q}$ is technically an equivalence class (as an element of $C(X) \times C(X) /\sim$),  the equation above holds for any representative. 
    Since $p, q \mapsto f_{p, q}(\Psi(p, \cdot)), g_{p,q}(\Psi(q, \cdot))$ is $C^k$ from $U \times U$ to $C(X) \times C(X) / \sim$ by \cref{lemma:regularity-of-integrand}, $G_\beta$ is $C^k$ on $U \times U$.
\end{proof}

\subsection{Non-degeneracy}
\label{subsec:non-degen}
In this section, we discuss sufficient conditions for the non-degeneracy condition (5. in Assumption \ref{assump:divergence}) to be satisfied. For this section, unless otherwise stated, $X = \overline{\Omega}$ is the closure of a bounded open set $\Omega \subset \R^d$, $\mu \in \mathcal{P}(X)$ and $j \in \Z_{\geq 0}$. 

To check whether the non-degeneracy condition is satisfied, we use results on the Hessian of the Sinkhorn divergence from \cite{Lavenant2025}, which involve reproducing kernel Hilbert spaces (RKHS).
We recall some facts from RKHS theory but refer to other references, e.g., \cite{steinwart2008support, muandet2017kernel}, for a more thorough exposition.
In what follows, $\mathcal{H}_k$ denotes the RKHS with the symmetric, positive definite kernel $k \colon \mathcal{X} \times \mathcal{X} \to \R$, where the input set $\mathcal{X}$ will be $X$ or $\R^n$, depending on context. For this paper, we will just work with RKHSs of real-valued functions.
Recall that by the Moore-Aronzajn theorem \cite{aronszajn1950theory}, whenever $k$ is symmetric and positive definite, there exists a RKHS $\mathcal{H}_k$ with kernel $k$ and it is unique.

A popular method in statistics and machine learning is to embed a probability measure $\mu$ into $\mathcal{H}_k$ via the map
\[ \mu \mapsto \int k(x, \cdot) d\mu(x), \]
known as the \emph{kernel mean embedding} \cite{ smola2007hilbert, muandet2017kernel}. By the Riesz representation theorem, the right-hand side is indeed in $\mathcal{H}_k$ whenever $f \mapsto \int f d\mu$ is a bounded linear functional on $\mathcal{H}_k$.
More generally, we follow \cite{simon2018kernel} and define the kernel mean embedding for any element of $\mathcal{H}_k^*$:

\begin{definition}[Kernel mean embedding]
    For a symmetric, positive definite kernel $k \colon \mathcal{X} \times \mathcal{X} \to \R$ and its corresponding RKHS $\mathcal{H}_k$, define the \emph{kernel mean embedding} $H_k \colon \mathcal{H}_k^* \to \mathcal{H}_k$ by
    \[ H_k[\nu](y) \vc = \langle \nu, k(\cdot, y)\rangle \]
    for each $y \in \mathcal{X}$ and $\nu \in \mathcal{H}_k^*$. One can check that $H_k[\nu] \in \mathcal{H}_k$ for each $\nu \in \mathcal{H}_k^*$ by observing that it is precisely the representation $\psi_\nu \in \mathcal{H}_k$ of $\nu \in \mathcal{H}_k^*$ given by the Riesz representation theorem:
    \[ \psi_\nu(y) = \langle \psi_\nu, k(\cdot, y) \rangle_{\mathcal{H}_k} =  \langle \nu, k(\cdot, y)\rangle  = H_k[\nu](y). \]
\end{definition}

As in \cite{Lavenant2025}, we are interested in RKHSs with the self-transport kernel:
\begin{definition}[Self-transport kernel]
    For $\mu \in \mathcal{P}(X)$, the \emph{self-transport kernel} $k_\mu \colon X \times X \to \R$ is given by
    \[
        k_\mu(x, y) \vc = \exp \left( \frac{f_{\mu, \mu}(x) + g_{\mu, \mu}(y) - \|x-y\|^2}{\beta} \right),
    \]
    where $f_{\mu, \mu}, g_{\mu, \mu}$ are Schr\"odinger potentials for $(\mu, \mu)$ (with regularization parameter $\beta > 0$). Following \cite{Lavenant2025}, for the rest of this section, we use the convention $f_{\mu, \mu} = g_{\mu, \mu}$, so that we can write
    \[
        k_\mu(x, y) = \exp \left( \frac{f_{\mu, \mu}(x) + f_{\mu, \mu}(y) - \|x-y\|^2}{\beta} \right).
    \]
\end{definition}

As the Gaussian kernel is a positive definite kernel, so is $k_\mu$ for each $\mu \in \mathcal{P}(X)$. As in \cite{Lavenant2025}, we will use $\mathcal{H}_\mu$ (instead of $\mathcal{H}_{k_\mu}$) to denote the RKHS with kernel $k_\mu$, and we will use $H_\mu$ to denote the kernel mean embedding associated to $\mathcal{H}_\mu$.
Since the Schr\"odinger potential $f_{\mu, \mu}$ is smooth, we have $k_\mu \in C^\infty(X \times X)$, which implies that $\mathcal{H}_\mu$ injects continuously into $C^j(X)$ (see \cite{simon2018kernel} or \cite[Appendix B]{Lavenant2025}) and $H_\mu$ is also well-defined as an operator from $C^j(X)^*$ to $C^j(X)$. 
\cite{Lavenant2025} additionally define $K_\mu \colon C(X) \to C^j(X)$  by
\[ K_\mu[\varphi](y) \vc= \int \varphi(x) k_\mu(x, y) d\mu(x) \]
for each $\varphi \in C(X)$ \cite[Definition 3.3]{Lavenant2025}.

We can now state the results from \cite{Lavenant2025} that we will use: as long as a curve of measures $\{\mu_t\}_{t \in (-\delta, \delta)}$ is continuously differentiable in $C^j(X)^*$ with the weak-* topology (what they call a \textit{$C^j$-perturbation}),
\begin{equation}
\label{eq:hess-sink}
    \lim_{t \to 0} \frac{S_\beta(\mu_0, \mu_t)}{t^2} = \frac{\beta}{2} \langle \dot{\mu}, (\mathrm{id} - K_\mu^2)^{-1} H_\mu [\dot{\mu}] \rangle \geq \frac{\beta}{2}\left\| H_\mu [\dot{\mu}] \right\|^2_{\mathcal{H}_\mu},
\end{equation}
where $\mu = \mu_0$ and $\dot{\mu} \in C^j(X)^*$ acts on $\varphi \in C^j(X)$ via
\begin{equation}
\label{eq:mu-dot-def}
\langle \dot{\mu}, \varphi \rangle = \frac{d}{dt} \langle \mu_t, \varphi\rangle \big|_{t = 0}.
\end{equation}
Observe that for any $\varphi \in C^j(X)$ and constant $c \in \R$, 
\[\langle \dot{\mu}, \varphi \rangle = \langle \dot{\mu}, \varphi + c \rangle,\] 
and so we can view $\dot{\mu}$ as acting on $C^j(X) / \R$, the quotient space 
obtained by quotienting out the constant functions in $C^j(X)$.
In particular, \cite{Lavenant2025} show that $\mathrm{id} - K_\mu^2 \colon C^j(X) / \R \to C^j(X) / \R$ is well-defined and has a bounded inverse. Hence $\langle \dot{\mu}, (\mathrm{id} - K_\mu^2)^{-1} H_\mu [\dot{\mu}] \rangle$ in (\ref{eq:hess-sink}) is well-defined.

It is now much more straightforward to check whether the non-degeneracy condition (5. in Assumption \ref{assump:divergence}) holds under Assumption \ref{assump:manifold-hyp} and Assumption \ref{assump:manifold-hyp-deform}.
\begin{proposition}
\label{prop:non-degen}
    Under Assumption \ref{assump:manifold-hyp}, 
    the non-degeneracy condition (5. in Assumption \ref{assump:divergence}) holds. 
    Under Assumption \ref{assump:manifold-hyp-deform}, if for all $p \in U$ and $i \in [m]$, there exists $\varphi \in C^1(X)$ such that
    \begin{equation}
    \label{eq:mu-dot-nontrivial}
        \int \frac{\partial \Psi}{\partial p_i} (p, x) \cdot \nabla \varphi (\Psi(p, x)) d\mu(x) \neq 0,
    \end{equation}
    the non-degeneracy condition also holds.
\end{proposition}

\begin{proof}
    Let us work under Assumption \ref{assump:manifold-hyp} first; the arguments for Assumption \ref{assump:manifold-hyp-deform} are analogous. Let $\gamma \colon (-\delta, \delta) \to \M$ be a smooth curve with $\gamma(0) = x \in \M$ and $\gamma'(0) = v \in T_x\M$ with $v \neq 0$. Set
    \[ \mu_t = F(\gamma(t)) = \sum_{j = 1}^n p_j \delta_{\iota_j(\gamma(t))} \]
    for $t \in (-\delta, \delta),$ and let $X = \overline{B_R(0)}$ for some $R > 0$ large enough that $\Ima ( \iota_j \circ \gamma) \subset X$ for all $j\in [n]$. We first check that $t \mapsto \mu_t$ is continuously differentiable in $C^1(X)^*$ with the weak-* topology, i.e, that
    \begin{equation}
    \label{eq:diff-in-weak-star}
    \dot{\mu}_t \vc = \lim_{s \to 0} \frac{\mu_{t+s} - \mu_t}{s}
    \end{equation}
    exists for each $t \in (-\delta, \delta)$ (where the limit is taken in the weak-* topology) and that $t \mapsto \dot{\mu}_t$ is continuous. This is similar to the argument in \cite[Remark 3.2]{Lavenant2025}: for any $t \in (-\delta, \delta)$ and $\varphi \in C^1(X)$, 
    \begin{align*}
    \lim_{s \to 0} \left\langle \frac{\mu_{t+s} - \mu_t}{s}, \varphi \right \rangle &=  \lim_{s \to 0} \sum_{j = 1}^n p_j \left(\frac{ \varphi(\iota_j(\gamma(t+s))) - \varphi(\iota_j(\gamma(t)))}{s} \right) \\
    &= \sum_{j = 1}^n p_j \nabla \varphi (\iota_j(\gamma(t))) \cdot (\iota_j \circ \gamma)'(t),
    \end{align*}
    so $\dot{\mu}_t \in C^1(X)^*$ exists and is given by
    \[ \left \langle \dot{\mu}_t, \varphi \right \rangle = \sum_{j = 1}^n p_j \nabla \varphi (\iota_j(\gamma(t))) \cdot (\iota_j \circ \gamma)'(t) \]
    for $\varphi \in C^1(X)$. For any $\varphi \in C^1(X)$, $t \mapsto \langle \dot{\mu}_t, \varphi \rangle$ is continuous, so $t \mapsto \dot{\mu}_t$ is continuous from $(-\delta, \delta)$ to $C^1(X)^*$ with the weak-* topology.
    
    Recall that in Assumption \ref{assump:manifold-hyp} we assumed that $\iota$ is a smooth embedding, so $(\iota \circ \gamma)'(t)$ is non-zero whenever $\gamma'(t)$ is non-zero. In particular, since we took $\gamma \colon (-\delta, \delta) \to \M$ to be such that $\gamma'(0) = v \neq 0$, there exists $j \in [n]$ such that $(\iota_j \circ \gamma)'(0)$ is non-zero, so $\dot{\mu}_0$ is non-trivial in $C^1(X)^*$. We claim that $\dot{\mu}_0$ is in fact non-trivial in $\mathcal{H}_{\mu_0}^*$, so that using (\ref{eq:hess-sink}) we can conclude that the non-degeneracy condition holds.
    
    Indeed, it suffices to observe that $\mathcal{H}_\mu$ is dense in $C^1(X)$ for any $\mu \in \mathcal{P}(X)$, as this would imply that any non-trivial element in $C^1(X)^*$ is also non-trivial in $\mathcal{H}_\mu^*$.
    The argument is similar to an argument in \cite[Appendix B]{Lavenant2025}, where they show that $\mathcal{H}_\mu$ is dense in $C(X)$ (i.e., $k_\mu$ is universal).
    Let $k_\beta (x, y) \vc = \exp \left(- \|x-y\|^2 / \beta \right)$ for each $x, y \in X$.
    Then for any $\varphi \in C^1(X)$, $c_1, c_2, \ldots, c_l \in \R$ and $y_1, \ldots, y_l \in X$, we have
    \begin{multline*}
    \left\| \sum_{i = 1}^l \frac{c_i k_\mu(y_i, \cdot)}{\exp\left( f_{\mu, \mu}(y_i)/\beta \right)} - \varphi \right\|_{C^1(X)} \leq 2 \left\|\exp \left( f_{\mu, \mu}/\beta \right)  \right\|_{C^1(X)}
         \left\| \sum_{i = 1}^l c_i k_\beta (y_i, \cdot ) - \frac{\varphi}{ \exp \left( f_{\mu, \mu}/\beta \right) } \right\|_{C^1(X)},
    \end{multline*}
    and the right hand side can be made arbitrarily small, given that the RKHS for the Gaussian kernel $k_\beta$ is dense in $C^1(X)$ \cite{simon2018kernel} (see \cref{rmk:universality-of-gaussian-kernel} below).

    The proof of the statement for Assumption \ref{assump:manifold-hyp-deform} is similar. Let $\{e_i\}_{i \in [m]}$ denote the standard coordinate basis for $\R^m$. With $U, \Psi, X$ as in Assumption \ref{assump:manifold-hyp-deform}, fix any $p \in U, i \in [m]$ and $\delta > 0$ such that $B_\delta(p) \subset U$, and let
    \[ \mu_t \vc = \Psi(p + te_i, \cdot)_{\#} \mu \]
    for $t \in (-\delta, \delta)$.
    Then, as in \cite[Remark 3.2]{Lavenant2025},
    by the chain rule and dominated convergence theorem, $\dot{\mu}_t = \lim_{s \to 0} (\mu_{t+s} - \mu_t)/s \in C^1(X)^*$ exists and is given by
    \[ \langle \dot{\mu}_t,\varphi  \rangle = \int \frac{\partial \Psi}{\partial p_i}(p + te_i, x) \cdot \nabla \varphi (\Psi (p+ te_i, x)) d\mu(x) \]
    for each $\varphi \in C^1(X)$, and so $t \mapsto \langle \dot{\mu}_t, \varphi \rangle$ is continuous for any $\varphi \in C^1(X)$. Therefore $t \mapsto \mu_t$ is continuously differentiable in $C^1(X)^*$ with the weak-* topology. The condition that (\ref{eq:mu-dot-nontrivial}) holds for some $\varphi \in C^1(X)$ ensures that $\dot{\mu}_0$ is non-trivial in $C^1(X)^*$. The density of $\mathcal{H}_{\mu_0}$ in $C^1(X)$ implies that $\dot{\mu}_0$ is non-trivial in $\mathcal{H}_{\mu_0}^*$ and we can conclude.
\end{proof}

\begin{remark}[On the universality of the Gaussian kernel]
    \label{rmk:universality-of-gaussian-kernel}
    It is well-known that the RKHS $\mathcal{H}_{k_{\beta}, X}$ with the Gaussian kernel $k_\beta(x, y) \vc = \exp \left( -\|x-y\|^2/\beta \right)$ is dense in $C(X)$ whenever $X \subset \R^d$ is compact \cite{steinwart2001influence}. We include $X$ in the subscript of $\mathcal{H}_{k_{\beta}, X}$ to emphasize that each element of $\mathcal{H}_{k_{\beta}, X}$ is a function over $X$. When a RKHS $\mathcal{H}_k$ is dense in $C(X)$, its kernel $k$ is called \emph{universal} \cite{steinwart2001influence, micchelli2006universal} or, more specifically, \emph{c-universal} \cite{sriperumbudur2010relation, simon2018kernel}.
    
    The terminology c-universal is used to emphasize that the RKHS $\mathcal{H}_k$ is dense in $C(X)$. One can also study whether a RKHS is dense in other function spaces.
    In particular, \cite{simon2018kernel} show that the RKHS $\mathcal{H}_{k_{\beta}, \R^d}$ (i.e., each element of $\mathcal{H}_{k_{\beta}, \R^d}$ is a function on $\R^d$) with the Gaussian kernel $k_\beta$ is dense in $C_0^1(\R^d)$. Using the terminology in \cite{simon2018kernel}, the Gaussian kernel is $c_0^1$-universal when the input set is $\R^d$. Here the space $C_0^1(\R^d)$ is the space of $C^1$ functions that vanish at infinity along with its first derivatives, equipped with the topology generated by the semi-norms $\|f\|_\alpha \vc = \sup_{x \in \R^d} |\partial^\alpha f(x)|, |\alpha| \leq 1$. We refer to \cite{simongabriel2019kerneldistributionembeddingsuniversal} for more details.

    What we need for the proof of \cref{prop:non-degen} is for $\mathcal{H}_{k_{\beta}, X}$ to be dense in $C^1(X)$ whenever $X = \overline{\Omega}$ for a bounded open convex set $\Omega \subset \R^d$. This follows from the density of $\mathcal{H}_{k_{\beta}, \R^d}$ in $C^1_0(\R^d)$ by using \cite[Theorem 6]{simon2018kernel}, which states that if a RKHS $\mathcal{H}_k$ with kernel $k$ injects continuously into a locally convex topological vector space of functions $\mathcal{F}$, $\mathcal{H}_k$ is dense in $\mathcal{F}$ \emph{if and only if} $k$ is strictly positive definite over $\mathcal{F}^*$, i.e.,
    \[ \left\| H_k[\nu] \right\|_{\mathcal{H}_k}^2 = \langle \nu, H_k[\nu]\rangle > 0 \]
    for all non-trivial $\nu \in \mathcal{F}^*$.
    (We used the forward direction in the proof of \cref{prop:non-degen}; the reverse direction follows from Hahn-Banach.) Thus the density of $\mathcal{H}_{k_{\beta}, \R^d}$ in $C^1_0(\R^d)$ implies that
    \begin{equation}
    \label{eq:gauss-kernel-spd}
    \langle \nu, H_{k_\beta}[\nu] \rangle > 0 
    \end{equation}
    for all non-trivial $\nu \in C^1_0(\R^d)^*$, where we recall that
    \[ H_{k_\beta}[\nu](y) = \langle \nu, k_\beta(\cdot, y) \rangle \]
    for $y \in \R^d$.
    Since any $f \in C^1(X)$ can be extended to some $\widetilde{f} \in C_0^1(\R^d)$ (see \cref{rmk:c-k-whitney-extension}), (\ref{eq:gauss-kernel-spd}) is true
    for all non-trivial $\nu \in C^1(X)^*$, and so $\mathcal{H}_{k_{\beta}, X}$ is dense in $C^1(X)$.
\end{remark}

\section{Examples}
\label{sec:examples}
In this section, we discuss two examples where $D$ is given by the Sinkhorn divergence. Unless otherwise specified, $\beta > 0$ is fixed. For the figures in this section, we set $\beta = 1$. The code for all of the figures in this paper can be found at \href{https://github.com/lzx23/graph_laplacians_w_sdiv}{https://github.com/lzx23/graph\_laplacians\_w\_sdiv}.

\subsection{Rotating particle}
\label{subsec:rotation}
Fix some $R > 0$. Consider 
$F \colon \S^1 \to \mathcal{P}(\R^2)$ 
given by
\begin{equation}
\label{eq:rotating-particle}
F(\theta) \vc = \frac{1}{2}\delta_{(0, 0)} + \frac{1}{2} \delta_{(R\cos \theta, R\sin \theta)}. 
\end{equation}
For convenience, let $\mu_\theta \vc = F(\theta)$ for $\theta \in [0, 2\pi).$
For each $\theta \in [-\pi, \pi),$ we have a closed-form formula for $\mathrm{OT}_\beta (\mu_0, \mu_\theta)$. 
\begin{claim}
\label{claim:rotating-particle-closed-form}
For each $\theta \in [-\pi, \pi)$, the optimal coupling for
\begin{equation}
\label{eq:OT-beta-rotating-particle}
    \mathrm{OT}_\beta(\mu_0, \mu_\theta) = \min_{\substack{\Pi \in \R^{2 \times 2},\\ \ind^\top \Pi =  [1/2, 1/2], \\ \Pi \ind = [1/2, 1/2]^\top}} (\Pi_{12}+\Pi_{21})R^2 + \Pi_{22}\left(2R\sin \left( \frac{\theta}{2} \right)\right)^2 + \beta \sum_{i, j = 1}^2 \Pi_{ij} \log \left(4 \Pi_{ij} \right),
\end{equation}
is given by
\begin{equation}
    \Pi_\theta^* = \begin{bmatrix}
        \frac{1}{2}-p(\theta) & p(\theta) \\
        p(\theta) & \frac{1}{2} -p (\theta)
    \end{bmatrix},
\end{equation}
where $p$ is given by
\begin{equation}
\label{eq:rotation-p-closed-form}
p(\theta) \vc= \frac{1}{2 \left( 1 + e^{R^2\cos(\theta)/\beta} \right)}.
\end{equation}
We then have
\begin{equation}
\label{eq:ot-beta-closed-form}
\mathrm{OT}_\beta(\mu_0, \mu_\theta) = \beta \log(2-4p(\theta)) + R^2(1-\cos(\theta)),
\end{equation}
from which one can compute
\[S_\beta(\mu_0, \mu_\theta) = \mathrm{OT}_\beta(\mu_0, \mu_\theta) - \frac{1}{2}\left( \mathrm{OT}_\beta(\mu_0, \mu_0) + \mathrm{OT}_\beta(\mu_\theta, \mu_\theta) \right) = \mathrm{OT}_\beta(\mu_0, \mu_\theta) - \mathrm{OT}_\beta(\mu_0, \mu_0) 
\]
\[
g^{(\beta)}_\theta\left( \frac{d}{d\theta}, \frac{d}{d\theta} \right) = \frac{1}{2}R^2(1-2p(0))\]
for each $\theta \in [-\pi, \pi).$
\end{claim}

We leave the proof of Claim \ref{claim:rotating-particle-closed-form} for Appendix \ref{sec:example-app}. 
For the plots in this subsection, we will use the closed-form formulae for $S_\beta$ and $g^{(\beta)}$, and, unless otherwise specified, $R = 2$ and $\beta = 1$. For the rest of the subsection, we will write $g$ for $g^{(1)}$.
\begin{remark}
\label{rmk:sinkhorn-algo}
In practice, Sinkhorn's algorithm (or variants thereof) are used to approximately solve the entropy-regularized optimal transport problem \cite{cuturi2013sinkhorn}. We did not include the error from Sinkhorn's algorithm in our analysis in \cref{sec:sink-div} (i.e., we have assumed that we can compute the Sinkhorn divergences exactly) and leave this for future work.
Nonetheless,
for this example and the parameters that we use, we find empirically that the default implementation of Sinkhorn's algorithm in the Python Optimal Transport toolbox \cite{flamary2021pot, flamary2024pot} provides a good enough approximation of the Sinkhorn divergence that the resulting approximation error for the discrete Laplacian is very small; see Appendix \ref{sec:example-app} for more details.
\end{remark}

\textbf{Fourth-order error.}
We first discuss \cref{prop:fourth-order-error}.
For this example, a more straightforward way of seeing that the fourth-order error in \cref{prop:fourth-order-error} holds is to observe that $\theta \mapsto S_\beta(\mu_0, \mu_\theta)$ is smooth and even, so that there is no third-order term in its Taylor expansion around $0$. We can also observe that \cref{prop:fourth-order-error} holds numerically in \cref{fig:rotating-particle-4th-order-error}.

\begin{figure}
    \centering
    \includegraphics[width=0.7\linewidth]{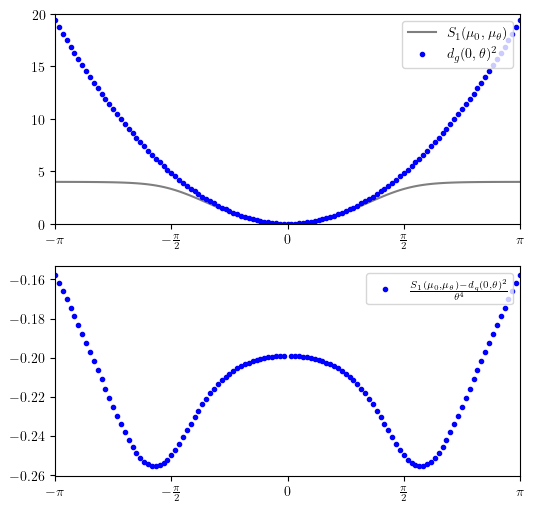}
    \caption{Visualization of the approximation error in \cref{prop:fourth-order-error} for the rotating particle example in \cref{subsec:rotation}.}
    \label{fig:rotating-particle-4th-order-error}
\end{figure}

\begin{figure}
    \centering
    \includegraphics[width=0.8\linewidth]{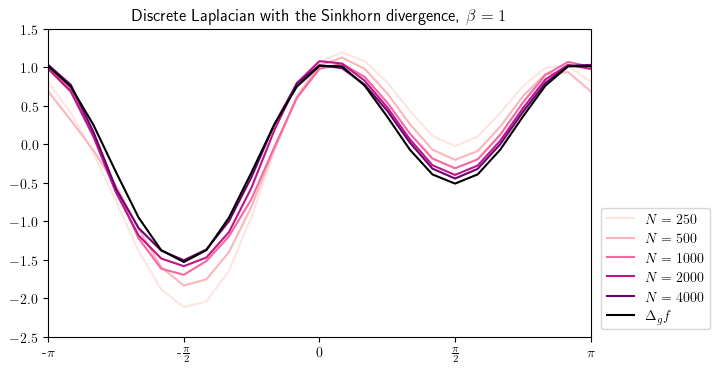}
    \caption{Example of pointwise convergence of the discrete Laplacian for the example in \cref{subsec:rotation}. We plot $\theta \mapsto \frac{2 \mathrm{vol}_g(\M)}{m_2}\frac{L^{(\eps_N, N)} f(\theta)}{N\eps_N^{m/2 + 1}}$ for $N = 250,\ 500,\ 1000,\ 2000,\ 4000$ and compare with $\Delta_g f$.}
    \label{fig:rotating-particle-gl-conv}
\end{figure}

\textbf{Convergence of the discrete Laplacian.}
Next, in \cref{fig:rotating-particle-gl-conv} we test the pointwise convergence of the discrete Laplacian constructed with $D(\theta, \varphi) \vc = S_1(\mu_\theta, \mu_\varphi)$ for $\theta, \varphi \in \S^1$. 
We draw $N$ samples $\theta_1, \ldots, \theta_N$ i.i.d. uniformly from $[0, 2\pi)$, take $h(x) = e^{-x/2}$ as the kernel function for constructing the weight matrix, set $\eps_N = 4N^{-1/3.01}$ and test using the function $f(\theta) = \sin(\theta) + \frac{1}{2}\cos(2\theta).$ For this choice of $h$, $m_2$ as defined in (\ref{eq:def_m_2}) can be computed to be $m_2 = (2\pi)^{m/2} = \sqrt{2\pi}$.
For convenience, we plot $\theta \mapsto \frac{2 \mathrm{vol}_g(\M)}{m_2}\frac{L^{(\eps_N, N)} f(\theta)}{N\eps_N^{m/2 + 1}}$, so that by \cref{thm:glconv}, for each $\theta \in \S^1$, the limit as $N \to \infty$ (almost surely) should be 
\begin{equation}
\label{eq:laplacian-f}
\Delta_g f(\theta) = -\frac{1}{g_0\left(\frac{d}{d\theta}, \frac{d}{d\theta} \right)} \frac{d^2f}{d\theta^2}(\theta).
\end{equation}
For our numerical experiments, we use the closed-form formula of $g$ to compute (\ref{eq:laplacian-f}) and $\mathrm{vol}_g(\M) = 2 \pi \sqrt{g_0\left(\frac{d}{d\theta}, \frac{d}{d\theta} \right)}$. 

In \cref{fig:rotating-particle-lap-eig}, \`a la the Laplacian eigenmap algorithm, we map each sample $\theta_i$ to $v^{(1)}_i, v^{(2)}_i$, where $v^{(1)}, v^{(2)}$ are the first two non-constant eigenvectors of $(\D^{(\eps_N, N)})^{-1}\mathcal{L}^{(\eps_N, N)}$, as introduced in \cref{subsec:gl-construction}. We color each point according to the value of $\theta_i$. As expected, we can see the underlying geometry of $(\S^1, g)$ in \cref{fig:rotating-particle-lap-eig}.
 
\begin{figure}
    \centering
    \includegraphics[width=\linewidth]{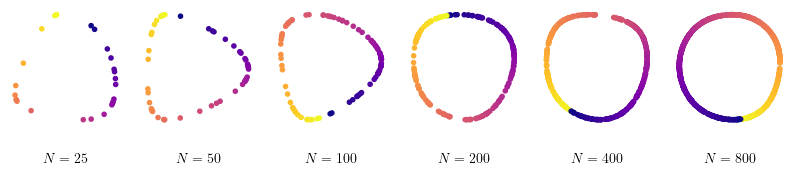}
    \caption{Embedding of samples into $\R^2$ using the first two non-constant eigenvectors $v^{(1)}, v^{(2)}$ of $(\D^{(\eps_N, N)})^{-1}\mathcal{L}^{(\eps_N, N)}$ for the example in \cref{subsec:rotation}. Each point is colored according to the value of $\theta_i$.}
    \label{fig:rotating-particle-lap-eig}
\end{figure}

\subsection{Dilation of two particles}
\label{subsec:dilation}
In this subsection, we discuss the following example from \cite[Section 7.2]{Lavenant2025}: for $x > 0$, set
\[ \mu_x \vc = \frac{1}{2} \delta_x + \frac{1}{2}\delta_{-x}. \]
Then, as shown in \cite[Lemma 7.5]{Lavenant2025}, the Sinkhorn divergence between $\mu_x$ and $\mu_y$ for $x, y> 0$ is given by the closed-form formula
\begin{equation}
\label{eq:dilation-sink-div}
S_\beta(\mu_x, \mu_y) = \beta \left( -\log \left(k_\beta(x, y) + k_\beta(x, -y) \right) + \frac{1}{2} \left( \log\left( 1 + k_\beta(x, -x)\right) + \log\left( 1 + k_\beta(y, -y)\right)\right)  \right), 
\end{equation}
where $k_\beta(x, y) \vc = \exp \left( -(x-y)^2/\beta \right)$ is the Gaussian kernel.
From (\ref{eq:dilation-sink-div}), one can show (see \cite[Eq. 7.3]{Lavenant2025}) that the associated Riemannian metric on $(0, \infty)$ is given by
\begin{equation}
\label{eq:dilation-riem-metric}
    g_x\left( \frac{d}{dx}, \frac{d}{dx} \right) \vc = \lim_{y \to x} \frac{S_\beta(\mu_x, \mu_y)}{(y-x)^2} = 1 + \frac{k(x)}{1+k(x)} \left( \frac{8x^2}{\beta(1+k(x))} - 2 \right)
\end{equation}
for $x \in (0, \infty)$,
where $k(x) \vc = k_\beta(x, -x)$. For convenience, we will use $S_\beta(x, y)$ to denote $S_\beta(\mu_x, \mu_y)$ and
$g(x)$ to denote $g_x\left( \frac{d}{dx}, \frac{d}{dx} \right)$
for $x, y \in (0, \infty)$. We also observe that the squared geodesic distance on $(0, \infty)$ with respect to the Riemannian metric $g$ in (\ref{eq:dilation-riem-metric}) is given by
\begin{equation}
\label{eq:geo-dist}
d_g(x, y)^2 = \left( \int_x^y \sqrt{g(z)}dz \right)^2 \end{equation}
for $x, y \in (0, \infty).$
\\~\\

\textbf{Fourth-order error.}
We now discuss how \cref{lemma:third-order-vanish} can be applied to this example. Fix any $a \in (0, \infty).$ Consider the Taylor expansions of $S_\beta(a, \cdot)$ and $d_g^2(a, \cdot)$ around $a$:
\[ S_\beta(a, a + h) = g(a) h^2 + \frac{1}{6} \frac{\partial^3 S_\beta}{\partial y^3} (a, a) h^3 + O(h^4) \]
\[ d_g^2(a, a+h) = g(a) h^2 + \frac{1}{6} \frac{\partial^3 [d_g^2]}{\partial y^3} (a, a) h^3 + O(h^4). \]
\cref{lemma:third-order-vanish} implies that 
\begin{equation}
\label{eq:4th-order-example}
\left| S_\beta(a, a+h) - d_g^2(a, a+h) \right| = O(h^4)
\end{equation}
as $h \to 0$, or equivalently,
\begin{equation}
\label{eq:equiv-of-3rd-derivs}
\frac{\partial^3 S_\beta}{\partial y^3} (a, a) = \frac{\partial^3 [d_g^2]}{\partial y^3} (a, a).
\end{equation}
One way to see this is to take the approach of \cref{prop:fourth-order-error}: consider some small enough $T > 0$ and a unit-speed, length-minimizing geodesic segment (with respect to $g$) $\gamma \colon [-T, T] \to (0, \infty)$ such that $\gamma(0) = a$. Now apply (\ref{eq:4th-order-error-2}) in \cref{lemma:third-order-vanish} to $f(s, t) \vc = S_\beta(\gamma(s), \gamma(t))$ to show that there exists $K > 0$ such that
\[ \left| S_\beta(a, \gamma(t)) - t^2 \right| \leq Kt^4  \]
for all $t \in (-T, T)$, or equivalently,
\[ \left| S_\beta\left(a, \gamma(t)\right) - d_g^2\left(a,  \gamma(t)\right) \right|  \leq K d_g^4\left(a, \gamma(t)\right) \]
for all $t \in (-T, T)$. 
For this example, upon fixing any $a_1, a_2 \in (0, \infty)$ such that $a_1 < a_2$, we have
\[ \left(\inf_{z \in [a_1, a_2]} g(z)\right) (x-y)^2 \leq d_g^2(x, y) \leq \left(\sup_{z \in [a_1, a_2]} g(z)\right) (x-y)^2 \]
for all $x, y \in [a_1, a_2]$, so 
(\ref{eq:4th-order-example}) holds.
Alternatively, we can directly compute the derivatives of $d_g^2$ with respect to the second variable:
\begin{align*}
    \frac{\partial [d_g^2]}{\partial y}(a, y) &= 2 \left( \int_a^y \sqrt{g(z)} dz \right) \sqrt{g(y)}\\
    \frac{\partial^2 [d_g^2]}{\partial y^2}(a, y) &= 2 \left( \left( \int_a^y \sqrt{g(z)} dz \right) \frac{g'(y)}{2\sqrt{g(y)}} + g(y) \right)
\end{align*}
for $y \in (0, \infty)$, so that
\begin{align*}
    \frac{\partial^3 [d_g^2]}{\partial y^3}(a, a) = 3g'(a).
\end{align*}
Recall that by construction 
$g(x) = \frac{1}{2} \frac{\partial^2 S_\beta}{\partial y^2}(x, x),$
so we have
\[ \frac{\partial^3 [d_g^2]}{\partial y^3}(a, a) = \frac{3}{2} \left( \frac{\partial^3 S_\beta}{\partial x \partial y^2}(a, a) + \frac{\partial^3 S_\beta}{\partial y^3}(a, a) \right) = \frac{\partial^3 S_\beta}{\partial y^3}(a, a), \]
where the second equality follows from \cref{lemma:third-order-vanish}.

We check (\ref{eq:4th-order-example}) numerically for $\beta = 1$ in \cref{fig:fourth-order-error}. In particular, we set $\delta = .005$ and plot  
\small
\[\frac{S_1(1, x) - \widehat{d}_g(1, x)^2}{(x-1)^3}, \frac{S_1(1, x) - g(1)(x-1)^2}{(x-1)^3}, \frac{S_1(1, x) - \widehat{d}_g(1, x)^2}{(x-1)^4}, \frac{S_1(1, x) - g(1)(x-1)^2}{(x-1)^4}\]
\normalsize
at $x = 1 + k \delta$ for $k \in \Z \setminus \{0\}$ such that $1 + k \delta \in [.9, 1.1]$. Here $\widehat{d}_g$ denotes a numerical approximation for the geodesic distance $d_g$ which we will describe below. In \cref{fig:fourth-order-error}, both $S_1$ and $g(1)$ are computed using their closed-form formulae, though we obtain similar results if we compute $S_1$ using Sinkhorn's algorithm as implemented in the Python Optimal Transport library \cite{flamary2021pot, flamary2024pot}; see Appendix \ref{sec:example-app}.

As $x \mapsto g(1)(x-1)^2$ is the second-order Taylor expansion of $x \mapsto S_1(1, x)$ around $1$, we expect $\frac{S_1(1, x) - g(1)(x-1)^2}{(x-1)^3}$ to converge as $x \to 1$, though $\frac{\left|S_1(1, x) - g(1)(x-1)^2\right|}{(x-1)^4}$ may go to infinity as $x \to 1$. This is indeed what we see in \cref{fig:fourth-order-error}. In contrast, if (\ref{eq:4th-order-example}) holds, we expect $\frac{S_1(1, x) - d_g(1, x)^2}{(x-1)^4}$ to stay bounded as $x \to 1$. This agrees with the behavior of $\frac{S_1(1, x) - \widehat{d}_g(1, x)^2}{(x-1)^4}$ in \cref{fig:fourth-order-error}.
\begin{figure}
    \centering
    \includegraphics[width=.85\linewidth]{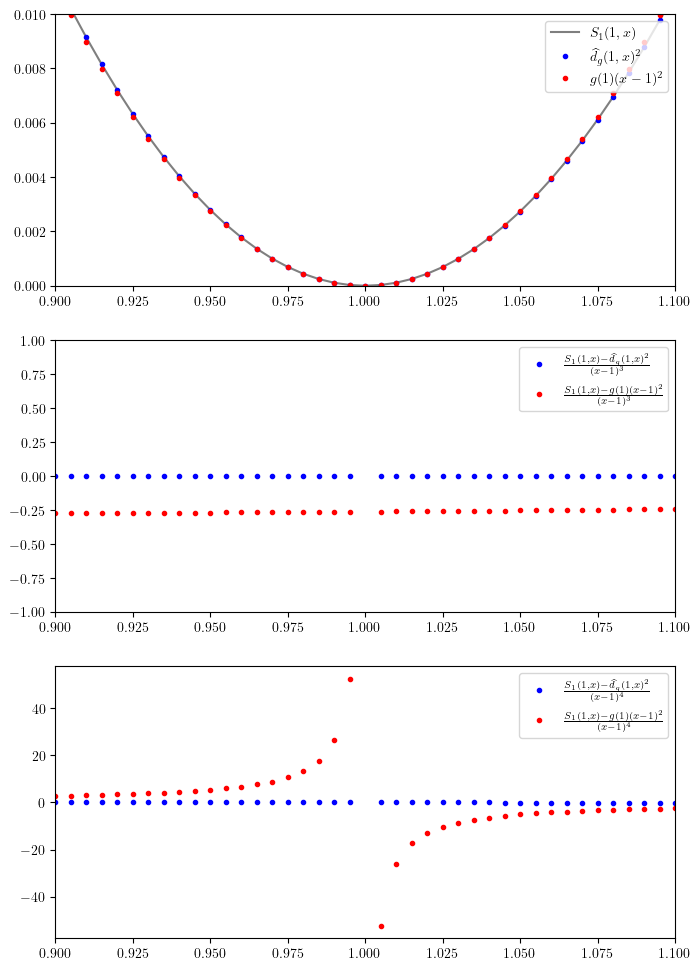}
    \caption{Visualization of the approximation error in (\ref{eq:4th-order-example}) for the dilation example in \cref{subsec:dilation}. Both $S_1$ and $g(1)$ are computed using their closed-form formulae, whereas the geodesic distance is approximated using Simpson's rule.}
    \label{fig:fourth-order-error}
\end{figure}

Last but not least, we describe how we compute $\widehat{d}_g$: for each $k \in \N$ such that $k\delta \in [.9, 1.1)$, $d_g(k\delta, (k+1)\delta)$ is approximated by applying Simpson's rule to (\ref{eq:geo-dist}), and these estimates are then summed up to approximate the geodesic distance $d_g(1, 1+k\delta)$ for $k \in \Z$ such that $ k\delta \in [-.1, .1]$.

\begin{figure}
    \centering
\includegraphics[width=.85\linewidth]{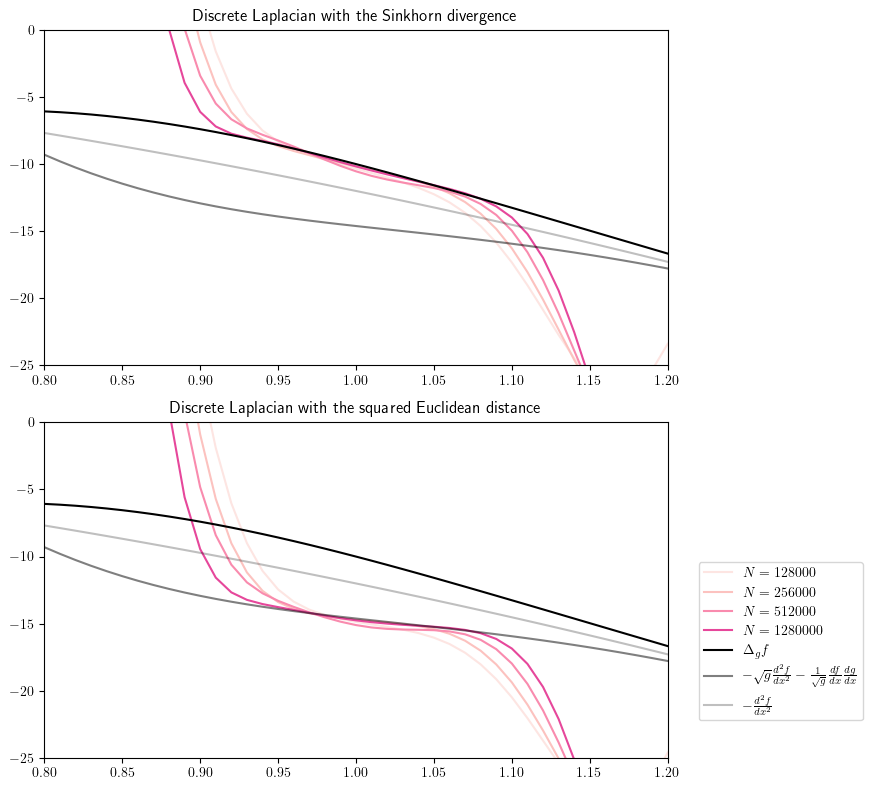}
    \caption{Examples of pointwise convergence of the discrete Laplacian for the example in \cref{subsec:dilation}. As in \cref{fig:rotating-particle-gl-conv}, we plot $x \mapsto \frac{2 \mathrm{vol}_g(\M)}{m_2}\frac{L^{(\eps_N, N)} f(x)}{N\eps_N^{m/2 + 1}}$, but in the top plot, $L^{(\eps_N, N)} f$ is constructed using $D(x, y)  = S_1(x, y)$ for $x, y \in \M$, whereas in the bottom plot, $L^{(\eps_N, N)} f$ is constructed using $D(x, y) = (x-y)^2$ for $x, y \in \M$.}
    \label{fig:dilation-gl-convergence}
\end{figure}
\begin{figure}
    \centering
    \includegraphics[width=\linewidth]{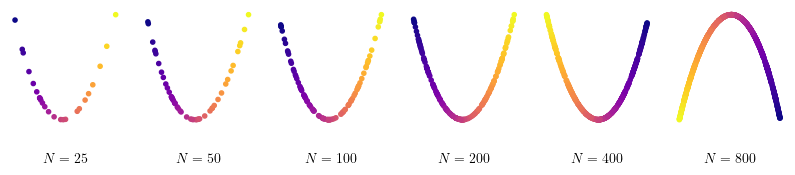}
    \caption{Embedding of samples into $\R^2$ using the first two non-constant eigenvectors $v^{(1)}, v^{(2)}$ of $(\D^{(\eps_N, N)})^{-1}\mathcal{L}^{(\eps_N, N)}$ (constructed with $D = S_1$) for the example in \cref{subsec:dilation}. Each point is colored according to the value of $x_i$.}
    \label{fig:dilation-lap-eig}
\end{figure}

\begin{remark}[On the approximation error for the geodesic distance]
Per the standard error bound for Simpson's rule (see, e.g., \cite{aschernumericalmethods}), for each $k \in \N$ such that $k\delta \in [.9, 1.1)$, our approximation $\widehat{d}_g(k\delta, (k+1)\delta)$ of $d_g(k\delta, (k+1)\delta)$ satisfies
\[ \left| \widehat{d}_g(k\delta, (k+1)\delta) - d_g(k\delta, (k+1)\delta)\right| \leq \frac{C\delta^5}{2880}, \]
where $C$ is an upper bound on the absolute value of the fourth derivative of $\sqrt{g}$ on $[.9, 1.1]$. 
Hence
\small
\[ \left| \frac{S_1(1, 1+k\delta) - \widehat{d}_g(1,1+k\delta)^2}{(k\delta)^3} - \frac{S_1(1, 1+k\delta) - d_g(1,1+k\delta)^2}{(k\delta)^3} \right| \leq \frac{C|k|\delta^5}{2880 |k\delta|^3} \left(2 d_g(1, 1+k\delta) + \frac{C|k|\delta^5}{2880}\right)  \]
\normalsize
and similarly
\small
\[ \left| \frac{S_1(1, 1+k\delta) - \widehat{d}_g(1,1+k\delta)^2}{(k\delta)^4} - \frac{S_1(1, 1+k\delta) - d_g(1,1+k\delta)^2}{(k\delta)^4} \right| \leq \frac{C|k|\delta^5}{2880 |k\delta|^4} \left(2 d_g(1, 1+k\delta) + \frac{C|k|\delta^5}{2880}\right) \]
\normalsize
for $k \in \Z \setminus \{0\}$ such that $ k\delta \in [-.1, .1]$. 
\end{remark}

\textbf{Convergence of the discrete Laplacian.} We also provide an example of the pointwise convergence of the discrete Laplacian in \cref{fig:dilation-gl-convergence}. For the rest of the subsection, we take $\beta = 1$. Although in previous sections we assumed that $\M$ does not have a boundary, for a manifold with boundary, the pointwise convergence of discrete Laplacians can be shown at points away from the boundary, following a similar proof. We use rejection sampling to sample $x_1, \ldots, x_N$ i.i.d. from the uniform distribution with respect to $g$ (as in (\ref{eq:dilation-riem-metric}), with $\beta = 1$) on $\M = [.8, 1.2]$, which has density 
\begin{equation}
\label{eq:sampling-density}
P(x) = \frac{\sqrt{g(x)}}{\mathrm{vol}_g (\M)}
\end{equation}
with respect to the Lebesgue measure on $[.8, 1.2]$. We set $h(x) = e^{-x/2}$ and $\eps_N = .15 N^{-1/3.01}$ and test using the function $f(x) \vc = x^4 - 8x$. As noted previously, with this choice of $h$, the constant $m_2$ in (\ref{eq:def_m_2}) is given by $m_2 = (2\pi)^{m/2} = \sqrt{2\pi}$. Similar to \cref{fig:rotating-particle-gl-conv}, we plot
\[x \mapsto \frac{2 \mathrm{vol}_g(\M)}{m_2}\frac{L^{(\eps_N, N)} f(x)}{N\eps_N^{m/2 + 1}},\] 
so that when $L^{(\eps_N, N)}f$ is constructed using $D(x, y) \vc= S_1(x, y)$, we expect convergence to
\[\Delta_g f \vc = -\frac{1}{\sqrt{g}} \frac{d}{dx} \left( \frac{1}{\sqrt{g}} \frac{df}{dx} \right) = -\frac{1}{g} \frac{d^2 f}{dx^2} + \frac{1}{2g^2} \frac{dg}{dx} \frac{df}{dx}
\]
away from the boundary. 
We use the closed-form formulae for $g$ and $\frac{dg}{dx}$ and approximate $\mathrm{vol}_g(\M)$ using our approximation for $d_g(.8, 1.2)$ described above.

For comparison, we also plot 
$x \mapsto \frac{2 \mathrm{vol}_g(\M)}{m_2}\frac{L^{(\eps_N, N)} f(x)}{N\eps_N^{m/2 + 1}}$ where $L^{(\eps_N, N)} f$ is constructed using the squared Euclidean distance $D(x, y) \vc= (x-y)^2$ instead of $S_1(x, y)$. Since we are sampling from the probability measure on $[0.8, 1.2]$ with density 
$P$ as in (\ref{eq:sampling-density}), we expect that the limiting function is instead
\[ \frac{2 \mathrm{vol}_g(\M)}{m_2} \left( \frac{m_2}{2} \left( -P \frac{d^2f}{dx^2} - 2 \frac{df}{dx} \frac{dP}{dx} ) \right)\right) = -\sqrt{g} \frac{d^2f}{dx^2} - \frac{1}{\sqrt{g}} \frac{df}{dx} \frac{dg}{dx}. \]
This is confirmed in the second subplot of \cref{fig:dilation-gl-convergence}.

In \cref{fig:dilation-lap-eig}, we map the $i$th sample to $(v^{(1)}_i, v^{(2)}_i)$, where $v^{(1)}, v^{(2)}$ are again the first two non-constant eigenvectors of the normalized graph Laplacian $(\D^{(\eps_N, N)})^{-1}\mathcal{L}^{(\eps_N, N)}$ constructed with $D = S_1$. Each point is colored according to the value of $x_i$.
For manifolds with boundary, we expect that for $x_1, \ldots, x_N$ sampled i.i.d. uniformly from $(\M, g)$, with high probability the eigenvectors of the graph Laplacian (appropriately scaled) converge to the Neumann eigenfunctions of the Laplace-Beltrami operator as $N \to \infty$, as is the case when $(\M, g)$ is a Riemannian submanifold of $\R^d$ and the squared Euclidean distance is used in the construction of the graph Laplacian \cite{coifman2006diffusion, singer2017spectral}. 

We also observe in \cref{fig:dilation-lap-eig} that the eigenvector $v^{(1)}$ gives us an ordering of the samples. Indeed, this is very similar to the spectral algorithm for seriation proposed by Atkins et al. \cite{atkins1998spectral}, which uses a Fiedler eigenvector (an eigenvector with the smallest eigenvalue among those orthogonal to the all ones vector) of the \textit{unnormalized} graph Laplacian to order samples. See \cite{warren2025principal} for an example of seriation using a graph Laplacian constructed from the squared Wasserstein-2 distance.

\section{Conclusion and future work}
\label{sec:concl}
In this work, we observe that a smooth symmetric divergence $D \colon \M \times \M \to \R_{\geq 0}$ satisfying Assumption \ref{assump:divergence} gives us an approximation of the squared geodesic distance $d_g^2$ (for $g$ as in (\ref{eq:hess-riem-metric})) up to \textit{fourth} order, in the sense that there exists $K > 0$ such that
 \[ \left| D(p, q) - d_g(p, q)^2 \right| \leq K d_g(p, q)^4 \]
for all $p, q \in \M$. This is a similar error bound to when $(\M, g)$ is a Riemannian submanifold of $\R^d$ and $D$ is given by the squared Euclidean distance, which is the typical setting for manifold learning. As an example of the kinds of results we expect to carry over from the Euclidean setting, we prove the pointwise convergence of graph Laplacians.

One motivation for this work was to better understand manifold learning with Sinkhorn divergences. To this end, we describe two models (Assumptions \ref{assump:manifold-hyp} and \ref{assump:manifold-hyp-deform}) and discuss how to check for the smoothness and non-degeneracy required for Assumption \ref{assump:divergence} under these models. We also study two concrete examples in \cref{sec:examples}. 

However, in our analysis, we fix the regularization parameter $\beta > 0,$ and we assume that we can compute the Sinkhorn divergence exactly. One direction for future work would be to study the dependence of various quantities (e.g., the Riemannian metric, $K$ from \cref{prop:fourth-order-error}) on $\beta$; see \cite{Lavenant2025} for an informal argument, for measures with density, on the convergence as $\beta \to 0$ of the Riemannian metric induced by $S_\beta$ to the Riemannian structure of Wasserstein-2 space. Another direction would be to do a more detailed analysis of manifold learning algorithms when the Sinkhorn divergences are computed via Sinkhorn's algorithm (or a variant thereof), which is used often in practice.

\section*{Acknowledgments}
This work was supported by
the National Science Foundation Graduate Research Fellowship Program under Grant No. DGE-2444107
and the Simons Foundation Math+X Investigator Award to Amit Singer. Any opinions,
findings, and conclusions or recommendations expressed in this material are those of the
author and do not necessarily reflect the views of the National Science Foundation. We thank Amit Singer for comments on an earlier draft and Gilles Mordant for helpful discussions.

\bibliographystyle{abbrv}
\bibliography{refs}

\appendix

\section{Proofs for \cref{sec:gl-conv}}
\label{sec:error-proof-app}
\begin{proof}[Proof of \cref{prop:fourth-order-error}]
    It suffices to prove that there exists $K, r > 0$ such that
    \begin{equation}
        \label{eq:4th-order-error-locally}
        \left| D(p, q) - d_g(p, q)^2 \right| \leq K d_g(p, q)^4
    \end{equation}
    for all $p,q \in \M$ such that $d_g(p, q) < r$. Indeed, since we assumed that $D$ is continuous and $\M$ is compact,
    \[ \sup_{ p, q \in \M, d_g(p, q) \geq r }  \frac{\left| D(p, q) - d_g(p, q)^2 \right|}{d_g(p, q)^4} < \infty. \]

    First consider any unit-speed, length-minimizing geodesic segment $\gamma \colon [-\delta, \delta] \to \M$ on $(\M, g)$. 
    Let $f(s, t) \vc = D(\gamma(s), \gamma(t))$ for $s, t \in [-\delta, \delta]$, and for each $s \in [-\delta, \delta],$ let $f_s \vc = f(s, \cdot)$. Assume $\delta$ is small enough so that $f \in C^4([-\delta, \delta] \times [-\delta, \delta])$. By Assumption \ref{assump:divergence}, $f$ is symmetric, $f \geq 0$ and $f(s, s) = 0$ for all $s \in [-\delta, \delta]$. Additionally, for any fixed $s \in (-\delta, \delta)$,
    \[ \lim_{t \to s} \frac{f(s, t)}{(t-s)^2} = \lim_{t \to s} \frac{D(\gamma(s), \gamma(t))}{(t-s)^2} = g_{\gamma(s)}(\gamma'(s), \gamma'(s)) = 1 \]
    since $\gamma$ is a unit-speed geodesic segment and we defined $g$ via
    \[ g_p = \frac{1}{2}\mathrm{Hess}_p (D(p, \cdot)) \]
    for all $p \in \M$.
    Thus $f_s''(s) = 2$ for all $s \in (-\delta, \delta)$, so \cref{lemma:third-order-vanish} implies that
    \begin{equation}
    \label{eq:pointwise-4th-order-error}
    \left| D(\gamma(0), \gamma(t)) - d_g(\gamma(0), \gamma(t))^2\right| = \left| D(\gamma(0), \gamma(t)) - t^2 \right| \leq \frac{\sup_{s \in (-\delta, \delta)} \left| f_0^{(4)}(s) \right|}{24} d_g(\gamma(0), \gamma(t))^4 
    \end{equation}
    for all $t \in (-\delta, \delta)$. Taking $q = \gamma(0)$ and $v = \gamma'(0)$, we have $f_0(t) = D(\gamma(0), \gamma(t)) = D(q, \exp_q(tv))$ for $t \in (-\delta, \delta)$. Defining $G_{q, v}(t) \vc= D(q, \exp_q(tv))$ for $(q, v) \in T\M$, it thus suffices to show that there exists $K, r > 0$ such that for all $q \in \M$ and $v \in T_q \M$ such that $g_q(v, v) = 1$,
    \begin{enumerate}
        \item $s, t \mapsto D(\exp_q(sv), \exp_q(tv))$ belongs to $C^4([-r, r] \times [-r, r])$, and
        \item $\left| G_{q, v}^{(4)} (t) \right| < K$ for all $t \in (-r, r)$.
    \end{enumerate}
    
    This follows from the compactness of $\M$ and is almost the same as \cite[Lemma 3.5]{xu2025manifold}. We provide the details for completeness.
    First fix any $p \in \M$. 
    Define $F \colon T\M \to \M \times \M$ by
    \begin{align*}
        F(q, v) \vc = (q, \exp_q(v)).
    \end{align*}
    (As $\M$ is compact, $F$ is well-defined on all of $T\M$ by Hopf-Rinow; if $\M$ were not compact, $F$ is still well-defined on a neighborhood of $(p, 0)$.)
    Then by the inverse function theorem, there exists an open neighborhood $U$ of $(p, 0)$ and an open neighborhood $V$ of $(p, p)$ such that $F |_U \colon U \to V$ is a diffeomorphism \cite[Theorem 3.7]{do1992riemannian}. 
    Making $U$ smaller if necessary, we can also assume $D$ is smooth on $V$, so $D \circ F$ is smooth on $U$.
    For small enough $r(p) > 0$, 
    the geodesic ball \[B(p, r(p)) \vc= \{q \in \M \colon d_g(p,q) < r(p)\}\]
    is strongly convex normal neighborhood of $p$ (see \cite[Ch. 3]{do1992riemannian}) and $\overline{B(p, r(p))} \times \overline{B(p, r(p))} \subset V$.

    For each $q \in \overline{B\left(p, \frac{r(p)}{2}\right)}$, $v \in T_q \M$ such that $g_q(v, v) = 1$ and $t \in \left[-\frac{r(p)}{2}, \frac{r(p)}{2}\right]$, 
    recall that by definition we have
    \[ G_{q, v}(t) \vc= D(F(q, tv)) = D(q, \exp_q(tv)). \]
    As $B(p, r(p))$ is strongly convex, $\exp_q(tv) \in \overline{B(p, r(p))}$ 
    for all $q \in \overline{B\left(p, \frac{r(p)}{2}\right)}$, $v \in T_q \M$ such that $g_q(v, v) = 1$ and $t \in \left[-\frac{r(p)}{2}, \frac{r(p)}{2}\right]$,
    and so for all such $q, v, t$, we have
    \[ (q, tv) \in F^{-1}\left( \overline{B(p, r(p))} \times \overline{B(p, r(p))} \right) \subset U. \]
    As $D \circ F$ is smooth on $U$ and $F^{-1}\left( \overline{B(p, r(p))} \times \overline{B(p, r(p))} \right)$ is compact, there exists $K_p > 0$ such that
    the fourth derivative
    \[ \left| G_{q, v}^{(4)}(t) \right| < K_p \]
    for all $q \in \overline{B\left(p, \frac{r(p)}{2}\right)}$, $v \in T_q \M$ such that $g_q(v, v) = 1$ and $t \in \left(-\frac{r(p)}{2}, \frac{r(p)}{2}\right)$.

    As $\M$ is compact, one can cover $\M$ with finitely many balls of the form $B\left(p, \frac{r(p)}{2} \right)$, and so there exist $K, r > 0$ such that $s, t \mapsto D(\exp_q(sv), \exp_q(tv))$ belongs to $C^4([-r, r] \times [-r, r])$ and $\left| G_{q, v}^{(4)} (t) \right| < K$ for all $t \in (-r, r)$ for all $(q, v)$ in the unit tangent bundle of $\M$.
\end{proof}

\begin{proof}[Proof of \cref{lemma:third-order-vanish}]

    For any $x, y \in (a, b)$, we have the Taylor expansions
    \begin{equation}
    \label{eq:taylor-exp-x}
        f(x, y) = \frac{1}{2} \frac{\partial^2 f}{\partial y^2}(x,x)(y-x)^2 + \frac{1}{6} \frac{\partial^3 f}{\partial y^3}(x, x)(y-x)^3 + \frac{1}{24}\frac{\partial^4 f}{\partial y^4}(x, s)(y-x)^4
    \end{equation}
    \begin{equation}
    \label{eq:taylor-exp-y}
    f(y, x) = \frac{1}{2} \frac{\partial^2 f}{\partial y^2}(y,y)(x- y)^2 + \frac{1}{6}\frac{\partial^3 f}{\partial y^3}(y,y)(x-y)^3+ \frac{1}{24} \frac{\partial^4 f}{\partial y^4}(y, t)(x-y)^4
    \end{equation}
    for some $s, t$ between $x$ and $y$.
   Since $f$ is symmetric, 
   we can equate (\ref{eq:taylor-exp-x}) and (\ref{eq:taylor-exp-y}) to obtain
   \[ \frac{\frac{\partial^2 f}{\partial y^2}(x,x) -  \frac{\partial^2 f}{\partial y^2}(y,y)}{2(y-x)} + \frac{1}{6} \left( \frac{\partial^3 f}{\partial y^3}(x,x) + \frac{\partial^3 f}{\partial y^3}(y, y) \right) + \frac{1}{24} \left( \frac{\partial^4 f}{\partial y^4}(x, s) - \frac{\partial^4 f}{\partial y^4}(y, t) \right)(y-x) = 0 \]
   for any $x, y \in (a, b)$ such that $x \neq y$.
    Fixing $x \in (a, b)$ and taking $y\to x$, we then have
    \[ -\frac{1}{2}\left( \frac{\partial^3 f}{\partial x \partial y^2}(x, x) + \frac{\partial ^3 f}{\partial y^3}(x, x)  \right) + \frac{1}{3} \frac{\partial^3 f}{\partial y^3}(x, x) = 0, \]
    which implies (\ref{eq:3rd-order-derivs}) upon rearranging.
    To prove (\ref{eq:4th-order-error-2}), observe that if we have
    \[ f_x''(x) = \frac{\partial^2 f}{\partial y^2}(x, x) = 2 \]
    for all $x \in (a, b)$, then
    \[ \frac{\partial^3 f}{\partial x \partial y^2}(x, x) + \frac{\partial^3 f}{\partial y^3}(x, x) = 0 \]
    for all $x \in (a, b)$. Using (\ref{eq:3rd-order-derivs}), we obtain
    \[ -\frac{1}{3}\frac{\partial^3 f}{\partial y^3}(x, x) + \frac{\partial^3 f}{\partial y^3}(x, x) = 0 \]
    for all $x \in (a, b)$, so $f_x'''(x) = 0$ for all $x \in (a, b)$ and (\ref{eq:4th-order-error-2}) follows from Taylor's theorem with the Lagrange remainder.
\end{proof}

\subsection{Pointwise convergence of the graph Laplacian (\cref{thm:glconv})}

\label{sec:gl-conv-app}

Throughout this section, assume that $D, g$ are given by Assumption \ref{assump:divergence} and that $h$ satisfies Assumption \ref{assump:lap-eig-kernel}. 
For each $\eps > 0$ and continuous $f \colon \M \to \R$, we follow \cite{coifman2006diffusion} and define $G_\eps f$ by
\[ G_\eps f (p) \vc = \frac{1}{\eps^{m/2}} \int_\M \K_\eps(p, q) f(q) dV_g(q) \]
for each $p \in \M$,
where $dV_g$ denotes the Riemannian volume form for $g$ and we recall from (\ref{eq:lap-eig-kernel-def}) that $\K_\eps \colon \M \times \M \to \R$ is given by
\[ \K_\eps(p, q)\vc = h \left( \frac{D(p, q)}{\eps} \right). \]
Using the notation from \cite[Lemma 8]{coifman2006diffusion}, we have the following proposition:
\begin{proposition}
\label{prop:integral-approx}
    Fix any $f \in C^3(\M)$. Then for each $p \in \M$,
    \begin{equation}
    \label{eq:G-eps-expansion}
    G_\eps f(p) = m_0 f(p) + \frac{\eps m_2}{2} \left( \omega(p) f(p) - \Delta_g f(p) \right) + O(\eps^{3/2}), 
    \end{equation}
    as $\eps \to 0^+$,
    where $\omega \colon \M \to \R$ depends on $g$ and $D$, and
    \[ m_0 \vc = \int_{\R^m} h(\|v\|^2) dv \]
    \[ m_2 \vc = \int_{\R^m} v_1^2 h(\|v\|^2) dv. \]
    If $f \in C^4(\M)$, the error can be improved to $O(\eps^2)$.
    Moreover, the multiplicative constant for the $O(\eps^{3/2})$ ($O(\eps^2)$ if $f \in C^4(\M)$) error can be made uniform across $p \in \M$.
\end{proposition}
This proposition is well-known when $(\M, g)$ is a compact Riemannian submanifold of $\R^d$ and $D$ is the squared Euclidean distance \cite[Lemma 8]{coifman2006diffusion} (see also \cite[Lemma A.5]{cheng2022convergence}).
\cref{prop:integral-approx} allows us to control the bias of the discrete Laplacian $L^{(\eps_N, N)}f$, and so \cref{thm:glconv} follows from standard concentration inequalities.

\begin{proof}[Proof of \cref{thm:glconv}]
We first show that $L^{(\eps_N, N)}f$ concentrates around its mean. This step is the same as if $(\M,g)$ were a compact Riemannian submanifold of Euclidean space and $D$ were the squared Euclidean distance, as it does not depend on any particular properties of the Euclidean distance (see \cite[Section 4.1]{kileel2021manifold} or \cite[Section 3]{xu2025manifold}).
For completeness, we briefly review the argument: for each $\eps > 0$ and $j \in \N$, upon defining the random variables $X_j^{(\eps)} \vc = \K_\eps(x, x_j)$, $Y_j^{(\eps)} \vc = \K_\eps(x, x_j)f(x_j)$, we have
\[ L^{(\eps, N)}f(x) = \sum_{j = 1}^N f(x) X_j^{(\eps)}-Y_j^{(\eps)},\]
and we can apply Hoeffding's inequality
since $f(x) X_1^{(\eps)} - Y_1^{(\eps)}, ..., f(x) X_N^{(\eps)}- Y_N^{(\eps)}$ are i.i.d. and $\left|f(x) X_j^{(\eps)} - Y_j^{(\eps)}\right| \leq 2c_0\|f\|_\infty$ almost surely for any $j \in \N$.
This gives us
\[ \mathbb{P}\left( \left| \frac{L^{(\eps, N)}f(x)}{N} - \E\left[ f(x) X_1^{(\eps)}-Y_1^{(\eps)} \right] \right| > \delta   \right) \leq 2e^{\textstyle-\frac{N\delta^2}{8c_0^2\|f\|_\infty^2}} \]
for any $\delta > 0$.
As $\eps_N = \Omega(N^{-\frac{1}{m+2 + \alpha}})$ for some $\alpha > 0$,
we have
\[ \sum_{N \in \N} \mathbb{P}\left( \frac{1}{\eps_N^{m/2 +1}} \left| \frac{L^{(\eps_N, N)}f(x)}{N} - \E\left[  f(x) X_1^{(\eps_N)}  - Y_1^{(\eps_N)}\right] \right| > \delta   \right) \leq \sum_{N \in \N} 2e^{\textstyle-\frac{N\left(\delta \eps_N^{m/2 + 1} \right)^2}{8c_0^2\|f\|_\infty^2}}< \infty, \]
for any $\delta > 0$, so one can conclude
\[  \mathbb{P}\left( \frac{1}{\eps_N^{m/2 +1}} \left| \frac{L^{(\eps_N, N)}f(x)}{N} - \E\left[  f(x) X_1^{(\eps_N)}  - Y_1^{(\eps_N)}\right] \right|\xrightarrow[]{n \to \infty} 0 \right) =  1\]
by using Borel-Cantelli and observing that $\delta > 0$ was arbitrary.
It thus remains to prove that
\[ \lim_{N \to \infty} \frac{1}{\eps_N^{m/2 +1}} \E\left[  f(x) X_1^{(\eps_N)} - Y_1^{(\eps_N)} \right] = \frac{m_2}{2}\left( P(x) \Delta_g f(x) - 2g_x( \grad_g f(x), \grad_g P (x) )\right). \]
This follows from expanding the definition of $\E\left[ f(x) X_1^{(\eps_N)} - Y_1^{(\eps_N)} \right]$ and using \cref{prop:integral-approx}:
\begin{align*}
    \frac{1}{\eps_N^{m/2}}\E\left[ f(x) X_1^{(\eps_N)} - Y_1^{(\eps_N)} \right] &=  \frac{1}{\eps_N^{m/2}}\left(f(x) \int \K_{\eps_N}(x, y) P(y) dV_g(y) - \int f(y) \K_{\eps_N}(x, y) P(y) dV_g(y)\right) \\
    &=   f(x) G_{\eps_N} P(x) - G_{\eps_N}[fP](x)  \\
    &= f(x)\left(m_0 P(x) + \frac{\eps_N m_2}{2} \left( \omega(x) P(x) - \Delta_g P(x) \right) + O\left(\eps_N^{3/2}\right)\right) - \\
    & \ \ \ \ \  \left(m_0 f(x)P(x) + \frac{\eps_N m_2}{2} \left( \omega(x) f(x)P(x) - \Delta_g (fP)(x) \right) + O\left(\eps_N^{3/2}\right)\right) \\
    &= \frac{\eps_N m_2}{2} \left( P(x) \Delta_g f(x) - 2g_x( \grad_g f(x), \grad_g P (x) ) \right) + O\left(\eps_N^{3/2}\right).
\end{align*}

\end{proof}

It remains to prove \cref{prop:integral-approx}. It is almost the same as the arguments in \cite{belkin2008towards, coifman2006diffusion}; again we provide the details for completeness.

\begin{proof}[Proof of \cref{prop:integral-approx}]
    Let $U$ be an open neighborhood of the diagonal $\Xi_\M \subset \M \times \M$ such that $D$ is smooth on $U$. 
    Let $\mathrm{inj}(\M, g)$ denote the injectivity radius of $(\M, g)$, and let $K$ be a positive number satisfying
    \begin{equation}
    \label{eq:fourth-order-error}
    \left| D(p, q) -d_g(p, q)^2\right| \leq Kd_g(p, q)^4 
    \end{equation}
    for all $p, q \in \M$. Such a $K > 0$ exists by \cref{prop:fourth-order-error}.
    Recall that we assumed that $\M$ is compact, so we also have $\mathrm{inj}(\M, g) > 0$.
    Choose
    \[ r \in \left( 0,  \min \left( \mathrm{inj}(\M, g), \inf_{(p, q) \in U^c} d_g(p, q), \sqrt{\frac{1}{2K}} \right) \right) \]
    if the complement $U^c$ of $U$ is non-empty, and 
     \[ r \in \left( 0, \min \left( \mathrm{inj}(\M, g), \sqrt{\frac{1}{2K}} \right) \right) \]
     if $U = \M \times \M$.

   \textbf{1. Approximation of $G_\eps f$ with an integral over a geodesic ball.}
    For any $p \in \M$, $\eps \in (0, r)$ and $f \in C^3(\M),$ we can approximate $G_\eps f(p)$ with an integral over the geodesic ball $B(p, r)$ of radius $r$:
    \[
    \left| G_\eps f(p) - \frac{1}{\eps^{m/2}} \int_{B(p, r)} \K_\eps(p, q) f(q) dV_g(q) \right| \leq  \frac{\mathrm{vol}_g(\M \setminus B(p, r)) \sup_{q \in \M \setminus B(p, r)} \left|\K_\eps(p, q) f(q) \right| }{\eps^{m/2}}.
    \]
    Since we assumed that $D$ is continuous and $D(p, q) = 0$ if and only if $p = q$,
    \[ D^* \vc = \inf_{p, q \in \M \colon d_g(p, q) \geq r} D(p, q) \]
    is strictly greater than 0.
    Per our assumptions on $h$ (Assumption \ref{assump:lap-eig-kernel}), 
    \[ 0 \leq \K_\eps(p, q) = h\left( \frac{D(p, q)}{ \eps } \right) \leq c_0 e^{-\frac{c D(p, q) }{\eps}} \leq c_0 e^{-\frac{cD^*}{\eps}}\]
    for all $p, q\in \M$ such that $d_g(p, q) \geq r$. Hence
    \begin{equation}
        \label{eq:error-1}
        \left| G_\eps f(p) - \frac{1}{\eps^{m/2}} \int_{B(p, r)} \K_\eps(p, q) f(q) dV_g(q) \right| \leq \frac{c_0 \mathrm{vol}_g(\M) e^{-\frac{cD^*}{\eps}}\|f\|_\infty }{\eps^{m/2}}
    \end{equation}
    for any $p \in \M$.
    
    \textbf{2. Taylor expansions and integration in normal coordinates.} Fix any $p \in \M$.
    We now approximate $\eps^{-m/2} \int_{B(p, r)} \K_\eps(p, q) f(q) dV_g(q)$ with $m_0 f(p) + \frac{\eps m_2}{2} \left(\omega(p)f(p) - \Delta_g f(p) \right)$, where $\omega \colon \M \to \R$ is to be determined. For this, we use normal coordinates around $p$ and Taylor expansions. For concreteness, we can fix an orthonormal basis $\{e_1, \ldots, e_m\}$ of $T_p \M$ to identify $T_p\M$ with $\R^m$. For any $q = \exp_p(v)$, $q\neq p$ with $\|v\| = d_g(p, q) < r$,  
    \begin{equation}
        \label{eq:kernel-tay-exp}
        h\left( \frac{D(p, q)}{\eps} \right) = h \left( \frac{\|v\|^2}{\eps} \right) + \left(\frac{D(p, q) - \|v\|^2}{\eps}\right) h'\left( \frac{\|v\|^2}{\eps} \right) + \frac{1}{2}\left(\frac{D(p, q) - \|v\|^2}{\eps}\right)^2 h''(s)
    \end{equation}
    for some $s$ between $D(p, q)/\eps$ and $\|v\|^2/\eps$, and moreover
    \[ \left| D(p, q) - \|v\|^2 \right| = \left| D(p, q) -d_g(p, q)^2\right| \leq Kd_g(p, q)^4 = K\|v\|^4. \]   
    By the choice of $r$, we have $r^2 \leq \frac{1}{2K}$, so the right-hand side can be bounded by
    \[ K\|v\|^4 \leq Kr^2 \|v\|^2 \leq \frac{\|v\|^2}{2} \]
    for all $v \in T_p \M$ such that $\|v\| < r$.
    By the exponential decay of the second derivative of $h$ in Assumption \ref{assump:lap-eig-kernel}, for $v, s$ as in (\ref{eq:kernel-tay-exp}),
    \[ \left|h''(s)\right| \leq c_2 e^{-cs} \leq c_2 e^{-\frac{c\left(\|v\|^2 - K \|v\|^4\right)}{\eps}} \leq c_2e^{-\frac{c\|v\|^2}{2\eps}}. \]
    Since
    \begin{equation} 
    \label{eq:order-of-moments}
    \int_{\R^m} \|v\|^k e^{-\frac{\alpha\|v\|^2}{\eps}}dv = O\left( \eps^{\frac{m+k}{2}} \right) 
    \end{equation}
    for any fixed $\alpha > 0$ and $k \in \N$ (use a change of variables $v \mapsto v/\sqrt{\eps}$), we have
    \begin{align*}
        \int_{\|v\| < r} \left(\frac{D(p, q) - \|v\|^2}{\eps}\right)^2 e^{-\frac{c\|v\|^2}{2\eps}} dv \leq \frac{K^2}{\eps^2} \int_{\R^m} \|v\|^8e^{-\frac{c\|v\|^2}{2\eps}}dv = O\left( \eps^{\frac{m+4}{2}} \right),
    \end{align*}
    so that approximating $h$ with its first-order Taylor expansion results in an $O(\eps^2)$-error:
    \small{
    \begin{align*}
        \frac{1}{\eps^{m/2}}& \int_{B(p, r)} \K_\eps(p, q) f(q) dV_g(q) \\
        &= \frac{1}{\eps^{m/2}} \int_{\|v\|< r} h\left( \frac{D(p, \exp_p(v))}{\eps} \right) f\left( \exp_p(v) \right)\sqrt{ \det \left[g_{\exp_p(v)}\right] } dv \\
        &= \frac{1}{\eps^{m/2}} \int_{\|v\| < r} \left( h \left( \frac{\|v\|^2}{\eps} \right) + \left(\frac{D(p, \exp_p(v)) - \|v\|^2}{\eps}\right) h'\left( \frac{\|v\|^2}{\eps} \right) \right) f(\exp_p(v)) \sqrt{ \det \left[g_{\exp_p(v)}\right] } dv + O(\eps^2).
    \end{align*}
    }\normalsize
    Here we use the shorthand $\left[g_{\exp_p(v)}\right]$ for the matrix
    \[ \left[g_{\exp_p(v)}\right] \vc = \left[ g_{\exp_p(v)} \left( \frac{\partial }{\partial v_i}, \frac{\partial }{\partial v_j} \right)\right]_{i, j} \]
    for each $v \in T_p \M, \|v\| \leq r$.
    We can also Taylor expand $v \mapsto f(\exp_p(v))$, $v \mapsto \sqrt{ \det\left[g_{\exp_p(v)}\right]}$ and $v \mapsto D(p, \exp_p(v))$ around $0$. Following the notation in \cite[Appendix B]{coifman2006diffusion}, we use $Q_{p, k}$ to denote a generic homogeneous polynomial of degree $k$ whose coefficients can depend on $p$ (i.e., the particular polynomial may change from line to line). Then as $v \to 0$, we have
    \begin{equation}
    \label{eq:f-tay-exp}
    f(\exp_p(v)) = f(p) + \langle \grad_g f(p), v\rangle + \frac{1}{2}\mathrm{Hess}_p f(v, v) + O(\|v\|^3)
    \end{equation}
    \begin{equation}
    \label{eq:jac-tay-exp}
    \sqrt{\det \left[g_{\exp_p(v)}\right]} = 1 - \frac{\mathrm{Ric}_p(v, v)}{6} + Q_{p, 3}(v) + O(\|v\|^4)
    \end{equation}
    \begin{equation}
        \label{eq:taylor-expand-D}
        D(p, \exp_p(v)) = \|v\|^2 + Q_{p, 4}(v) + Q_{p, 5}(v) + O(\|v\|^6)
    \end{equation}
    for $v \in T_p\M$ such that $\|v\| < r$. Here $\mathrm{Ric}_p$ denotes the Ricci curvature tensor at $p$. Notably, $v \mapsto D(p, \exp_p(v))$ has no third-order term in its Taylor expansion because of \cref{prop:fourth-order-error}. 
    If we additionally have $f \in C^4(\M)$, then we have
    \begin{equation}
    \label{eq:f-tay-exp-c4}
    f(\exp_p(v)) = f(p) + \langle \grad_g f(p), v\rangle + \frac{1}{2}\mathrm{Hess}_p f(v, v) + Q_{p, 3}(v) + O(\|v\|^4).
    \end{equation}
    Observing that odd functions integrate to 0 and using the exponential decay of $h, h'$ from Assumption \ref{assump:lap-eig-kernel} with (\ref{eq:order-of-moments}), we obtain
    \small
    \begin{align*}
        \frac{1}{\eps^{m/2}} & \int_{\|v\| < r}  h \left( \frac{\|v\|^2}{\eps} \right)  f(\exp_p(v)) \sqrt{ \det \left[g_{\exp_p(v)}\right] } dv \\
        &= \frac{1}{\eps^{m/2}} \int_{\|v\| < r} h \left( \frac{\|v\|^2}{\eps} \right) \left( f(p)  \left( 1 - \frac{\mathrm{Ric}_p(v, v)}{6} \right) + \frac{1}{2}\mathrm{Hess}_p f(v, v) \right)  dv + O(\eps^{3/2}),
    \end{align*}
    \normalsize
    where the error can be improved to $O(\eps^2)$ if $f \in C^4(\M)$ and we can use (\ref{eq:f-tay-exp-c4}),
    and also
    \small
    \begin{align*}
    \frac{1}{\eps^{m/2}} \int_{\|v\| < r}  &\left(\frac{D(p, \exp_p(v)) - \|v\|^2}{\eps}\right)  h'\left( \frac{\|v\|^2}{\eps} \right)  f(\exp_p(v)) \sqrt{ \det \left[g_{\exp_p(v)}\right] } dv \\
    &= \frac{1}{\eps^{m/2}} \int_{\|v\| < r} \left(\frac{Q_{p, 4}(v)}{\eps}\right) h'\left( \frac{\|v\|^2}{\eps} \right) f(p) dv + O(\eps^{2}).
    \end{align*}
    \normalsize
    Again using the exponential decay of $h$ and $h'$, we can approximate the above integrals over $\{ v \in T_p \M \colon \|v\|< r\}$ with their respective integrals over $\R^m$ with an error smaller than $O(\eps^2)$, so that altogether we obtain
    \small
    \begin{align*}
    \frac{1}{\eps^{m/2}}& \int_{B(p, r)} \K_\eps(p, q) f(q) dV_g(q) \\
    &= \frac{1}{\eps^{m/2}} \int_{\R^m} h \left( \frac{\|v\|^2}{\eps} \right) \left( f(p)  \left( 1 - \frac{\mathrm{Ric}_p(v, v)}{6} \right) + \frac{1}{2}\mathrm{Hess}_p f(v, v) \right) + \frac{f(p)Q_{p, 4}(v) h'\left( \frac{\|v\|^2}{\eps} \right)}{\eps}    dv + O(\eps^{3/2}) \\
    &= f(p) \left( m_0 + \eps \int_{\R^m} \frac{- h(\|v\|^2)\mathrm{Ric}_p(v, v)}{6}   + Q_{p, 4}(v) h'(\|v\|) dv \right) + \frac{\eps}{2} \int_{\R^m} h(\|v\|^2) \mathrm{Hess}_p f(v, v) dv + O(\eps^{3/2})
    \end{align*}
    \normalsize
    upon the change of variables $v \mapsto v/\sqrt{\eps}$,
    where again the error can be improved to $O(\eps^2)$ if $f \in C^4(\M)$. (\ref{eq:G-eps-expansion}) now follows from the equality
    \[ \int_{\R^m} h(\|v\|^2) \mathrm{Hess}_p f(v, v) dv = -m_2 \Delta_g f(p), \]
    which holds from diagonalizing $\mathrm{Hess}_p f$ and using $\Delta_g f(p) = - \mathrm{Tr}(\mathrm{Hess}_p f)$.
    
    The multiplicative constant for the error can be made uniform over $p \in \M$ by observing that the dependence of the error on $p$ is in terms of finitely many derivatives of 
    \[v \mapsto f(\exp_p(v)), \sqrt{ \det \left[g_{\exp_p(v)}\right] }, D(p, \exp_p(v)),\] which can be uniformly bounded on the compact set $\{ (p, v) \in T \M \colon \|v\| \leq r\}.$
    
\end{proof}

\section{Proofs for \cref{sec:sink-div}}
\label{sec:sink-div-app}
For this section, we use $\eps$ to denote an infinitesimal, not the bandwidth parameter for the graph Laplacian in \cref{subsec:gl-construction}.
\begin{proof}[Proof of \cref{lemma:T_beta-smooth}]
For convenience, let $k_\beta(x, y) \vc = \exp\left( - \|x-y\|^2/\beta \right).$
    Similar to \cite[Lemma 3.13]{Lavenant2025}, it suffices to prove that $A_\beta \colon C^k(X) \times U \to C^k(X)$ given by
        \[ [A_\beta(g, p)](y) \vc= \int g(x) k_\beta(x, y) d\mu_p(x) \]
    is $C^k$. The proof of \cite[Lemma 3.13]{Lavenant2025} establishes the continuous differentiability of $A_\beta$ when $U = (-\delta, \delta)$ for some $\delta > 0$ and $\{\mu_t\}_{t \in (-\delta, \delta)}$ is what they call a $C^k$-perturbation (continuously differentiable in $C^k(X)^*$ with the weak-* topology). Although we could use similar arguments to establish the continuous differentiability of $A_\beta$ in our setting, under Assumption \ref{assump:manifold-hyp-deform}, a more direct proof is possible. Moreover, since this approach generalizes to higher order derivatives and elucidates the dependence of the regularity of $A_\beta$ on the regularity of $\Psi$ with respect to the first variable, we will take this approach.
    
    To start, let us define
    \begin{equation}
    \label{eq:H-def}
    H_g(p, x, y) \vc = g(\Psi(p, x)) k_\beta(\Psi(p, x), y)
    \end{equation}
    \begin{equation}
    \label{eq:H-alpha-def}
    H_{g, \alpha}(p, x, y) \vc = g(\Psi(p, x))\partial_y^\alpha k_\beta(\Psi(p, x), y) = \partial_y^\alpha H_g(p, x, y)
    \end{equation}
    for each $p \in U$, $x, y \in X$, $g \in C^k(X)$ and multi-index $\alpha$ with $|\alpha| \leq k$, so that
    \[ [A_\beta(g, p)](y) = \int H_g(p, x, y) d\mu(x) \]
    for each $g \in C^k(X)$ and $p \in U$.
    One can then check via the dominated convergence theorem that $A_\beta(g, p)$ is indeed in $C^k(X)$ for each $g \in C^k(X)$ and $p \in U$, with
    \begin{equation}
    \label{eq:deriv-A-beta-g-p}
    \partial^{\alpha}[A_\beta(g, p)](y) =\int g(\Psi(p, x)) \partial_y^{\alpha}k_\beta(\Psi(p, x), y) d\mu(x)  = \int H_{g, \alpha}(p, x, y) d\mu(x)
    \end{equation}
    for each multi-index $\alpha$ with $|\alpha| \leq k$.

\bigskip

    \textbf{Continuity of $A_\beta$.}
    We first show that $A_\beta$ is continuous from $C^k(X) \times U$ to $C^k(X)$. Suppose $g_n \to g$ in $C^k(X)$ and $p_n \to p$ in $U$. For any multi-index $\alpha$ with $|\alpha| \leq k$, since $H_{g, \alpha}$
    is continuous on $U \times X \times X$, it is uniformly continuous on
    $\overline{B_r(p)} \times X \times X$ for some small enough $r> 0$. Hence for any $\eps > 0$, there exists $\delta > 0$ such that whenever $|p_n - p| < \delta$,
    \[ \left\| A_\beta(g, p_n) - A_\beta(g, p) \right\|_{C^k(X)} \leq \eps. \]
    Moreover, for any $n \in \N$, one can bound
    \[ \left\| A_\beta(g_n, p_n) - A_\beta(g, p_n) \right\|_{C^k(X)} \leq \|g_n - g\|_{C^0(X)} \sum_{|\alpha| \leq k} \sup_{x, y \in X} \left| \partial_y^{\alpha} k_\beta(x, y) \right|,   \]
    so $\|A_\beta(g_n, p_n) - A_\beta(g, p)\|_{C^k(X)} \to 0$ as $n \to \infty$.

    \bigskip
    
    If $k \geq 1$, by \cite[Proposition 3.5]{lang1999diffgeo}, to show that $A_\beta$ is $C^k$, it suffices to check that the partial derivatives \[D_1 A_\beta \colon C^k(X) \times U \to \mathcal{B}(C^k(X), C^k(X))\] \[D_2 A_\beta \colon C^k(X) \times U \to \mathcal{B}(\R^m, C^k(X))\] exist and are $C^{k-1}$. 

    \bigskip

    \textbf{Existence and continuity of $D_1 A_\beta$ when $k \geq 1$.}
    We discuss $D_1 A_\beta$ first. 
    Upon fixing any $p \in U$, the map $g \mapsto A_\beta(g, p)$ is a bounded linear operator from $C^k(X)$ to itself, so
    for any $g \in C^k(X)$,
    \[D_1 A_\beta(g, p)= \left(f \mapsto A_\beta(f, p)\right).\] 
    To show that $D_1 A_\beta$ is continuous, it suffices to check that if $p_n \to p$ in $U$, then the map $f \mapsto A_\beta(f, p_n)$ converges to $f \mapsto A_\beta(f, p)$ in the $\mathcal{B}(C^k(X), C^k(X))$ norm. 
    The continuity of $A_\beta$ gives us that for each fixed $f \in C^k(X)$, we have $A_\beta(f, p_n) \to A_\beta(f, p)$ in $C^k(X)$; we would now like to show that this convergence is uniform across $f \in C^k(X)$ such that $\|f\|_{C^k(X)} = 1$. 
    Using (\ref{eq:deriv-A-beta-g-p}) and the definition of $H_{f, \alpha}$ from (\ref{eq:H-alpha-def}), we have
    \begin{align*}
       \partial^{\alpha}[A_\beta(f, p_n)](y)& - \partial^{\alpha} [A_\beta(f, p)](y) \\&= \int f(\Psi(p_n, x))\partial_y^\alpha k_\beta(\Psi(p_n, x), y) - f(\Psi(p, x))\partial_y^\alpha k_\beta(\Psi(p, x), y) d\mu(x) \\
       &= \int H_{f, \alpha}(p_n, x, y) - H_{f, \alpha}(p, x, y) d\mu(x),
    \end{align*}
    for any $f \in C^k(X)$, $n \in \N$ and multi-index $\alpha$ such that $|\alpha| \leq k$. Let $r > 0$ be such that $\overline{B_r(p)} \subset U$.
    By the mean value theorem and the chain rule, for each multi-index $\alpha$ such that $|\alpha| \leq k$,
    there exists $C > 0$ (which can depend on the derivatives of $k_\beta$ and the first derivatives of $\Psi$ with respect to the first variable) such that for large enough $n$,
    \[ \sup_{x, y \in X} \left| H_{f, \alpha}(p_n, x, y) - H_{f, \alpha}(p, x, y) \right| \leq \left\|p_n - p \right\| \sup_{\substack{q \in \overline{B_r(p)}, \\x,y \in X}} \left\| \nabla_q H_{f, \alpha}(q, x, y) \right\| \leq C \left\|p_n - p \right\| \|f\|_{C^1(X)}. \]
    It follows that there exists $C' > 0$ such that
    \[ \|A_\beta(f, p_n) - A_\beta(f, p)\|_{C^k(X)} \leq C' \left\| p_n - p\right\| \|f\|_{C^1(X)} \]
    for large enough $n$, so the map $f \mapsto A_\beta(f, p_n)$ does indeed converge to $f \mapsto A_\beta(f, p)$ in the $\mathcal{B}(C^k(X), C^k(X))$ norm for $k \geq 1$. Thus $D_1 A_\beta$ is continuous. 

    \bigskip
    \textbf{Differentiability of $D_1A_\beta$ when $k \geq 2$.}
    Now suppose $k \geq 2$. We want to show that $D_1 A_\beta$ is $C^{k-1}$. Again we look at its partial derivatives: we can see that $D_1^2 A_\beta$ is identically 0 (since $D_1 A_\beta(g, p)$ does not depend on $g$), whereas the existence and continuity of
    \[ D_2 D_1 A_\beta \colon C^k(X) \times U \to \mathcal{B}\left( \R^m, \mathcal{B}(C^k(X), C^k(X)) \right) \]
    follows from a Taylor expansion. For simplicity, let us consider when $m=1$ (i.e., $U \subset \R$); in this case, we can identify $\mathcal{B}(\R, \mathcal{B}(C^k(X), C^k(X)))$ with $\mathcal{B}(C^k(X), C^k(X))$.
    The case $m > 1$ is similar.
    Fix any $p \in U \subset \R$, and define the operator $\Lambda_p \colon C^k(X) \to C^k(X)$ by
    \begin{equation}
    \label{eq:Lambda-p-def}
    \Lambda_p(f) \vc = \left(y \mapsto \int \partial_p H_f(p, x, y) d\mu(x) \right)
    \end{equation}
    for each $f \in C^k(X)$,
    where we recall that $H_f$ is given by (\ref{eq:H-def}).
    One can use the dominated convergence theorem to check that we indeed have $\Lambda_p(f) \in C^k(X)$ for each $f \in C^k(X)$ and $\Lambda_p \in \mathcal{B}(C^k(X), C^k(X))$.
    Moreover, there exists $C > 0$ such that
    \begin{align*}
        \left\|A_\beta(f, p + \eps) - A_\beta(f, p) - \eps \Lambda_p(f)  \right\|_{C^k(X)}  \leq C \|f\|_{C^2(X)} |\eps|^2
    \end{align*}
    for any $f \in C^k(X)$ and $\eps \in \R$ with small enough $|\eps|$. Indeed, since $\Psi$ is twice differentiable with respect to the first variable, we can apply Taylor's theorem with the Lagrange remainder to $H_f(\cdot, x, y)$ and $H_{f, \alpha}(\cdot, x, y)$ for each $x, y \in X$ and multi-index $|\alpha| \leq k$. The existence of such a $C > 0$ then follows from the first two derivatives of $\Psi$ with respect to $p$ being continuous and using that $\partial_p \partial_y^\alpha H_f = \partial_y^\alpha \partial_p H_f$ for each multi-index $|\alpha| \leq k$. Thus
    \[ \lim_{\eps \to 0} \frac{\left\| D_1 A_\beta(g, p+\eps) - D_1 A_\beta(g, p) - \eps \Lambda_p  \right\|_{\mathcal{B}(C^k(X), C^k(X))}}{|\eps|} = 0 \]
    for any $g \in C^k(X)$ and $p \in U$, so $D_2 D_1 A_\beta(g, p)$ exists and is given by $\Lambda_p$, upon identifying 
   $\mathcal{B}(\R, \mathcal{B}(C^k(X), C^k(X)))$ with $\mathcal{B}(C^k(X), C^k(X))$.
   That $D_2 D_1 A_\beta$ is continuous follows from similar arguments as the continuity of $D_1 A_\beta$. The arguments for higher-order derivatives are also similar.
    \bigskip
    
    \textbf{Existence and regularity of $D_2 A_\beta$ when $k \geq 1$.}
    Again, for simplicity, consider $m = 1$, so $U \subset \R$. Fixing any $g \in C^k(X)$ and $p \in U$,
    we have
    \[
        \lim_{\eps \to 0}\frac{\left\| A_\beta(g, p+\eps) - A_\beta(g, p) - \eps \Lambda_p(g) \right\|_{C^k(X)}}{|\eps|} = \lim_{\eps \to 0} \left\| \frac{A_\beta(g, p+\eps) - A_\beta(g, p)}{\eps} - \Lambda_p(g) \right\|_{C^k(X)} = 0
    \]
    by the mean value theorem, the equality $\partial_p \partial_y^\alpha H_g = \partial_y^\alpha \partial_p H_g$ for each multi-index $|\alpha| \leq k$ and a uniform continuity argument. Therefore $D_2 A_\beta \colon C^k(X) \times U \to \mathcal{B}(\R, C^k(X))$ exists and is given by
    \[ D_2A_\beta(g, p) = \left( \eps \mapsto \eps \Lambda_p(g) \right) \]
    for each $g \in C^k(X)$ and $p \in U$.
    Upon making the identification $\mathcal{B}(\R, C^k(X)) \cong C^k(X)$, we see that $D_2 A_\beta$ is continuous:
    if $g_n \to g$ in $C^k(X)$ and $p_n \to p$ in $U$,
    then there exists $C > 0$ (which can depend on $\partial_p\Psi$ and the derivatives of $k_\beta$, but not $n$) such that
    \[ \left\| \Lambda_{p_n} (g_n) - \Lambda_{p_n}(g) \right\|_{C^k(X)} \leq C \|g_n - g\|_{C^1(X)} \]
    for all $n \in \N$,
    and also
    \[ \left\| \Lambda_{p_n} (g) - \Lambda_{p}(g) \right\|_{C^k(X)} \to 0 \]
    as $n \to \infty$ by a uniform continuity argument. Hence
    \[ \left \| \Lambda_{p_n} (g_n) - \Lambda_p (g) \right\|_{C^k(X)} \to 0 \]
    as $n \to \infty$, i.e., $g, p \mapsto \Lambda_p(g)$ is continuous from $C^k(X) \times U$ to $C^k(X)$. The arguments for establishing that $D_2 A_\beta$ is in fact $C^{k-1}$ when $k \geq 2$ are similar to those already discussed.
\end{proof}

\begin{proof}[Proof of \cref{lemma:regularity-of-integrand}]
    It suffices to check that
    \begin{align*}
        H \colon U \times U \times U \times U &\to \widetilde{C}(X)  \\
        p, q, r, s & \mapsto f_{p, q}(\Psi(r, \cdot)), g_{p, q}(\Psi(s, \cdot))
    \end{align*}
    is a $C^k$ map. 
    
    We first establish the continuity of $H$. Suppose $(p_n, q_n, r_n, s_n) \to (p, q, r, s)$ in $U \times U \times U \times U$.
    Then for any $n \in \N$,
    \[ \left\| H(p_n, q_n, r_n, s_n) - H(p, q, r_n, s_n) \right\|_{\widetilde{C}(X) } \leq \left\| \left(f_{p_n, q_n}, g_{p_n, q_n} \right) - \left( f_{p, q}, g_{p, q}\right) \right\|_{\widetilde{C}(X) }, \]
    and the right-hand side goes to 0 as $n \to \infty$ by 
    \cref{lemma:schro-potentials-regularity}. Moreover,
    \begin{multline*}
        \left\| H(p, q, r_n, s_n) - H(p, q, r, s)  \right\|_{\widetilde{C}(X) }  \\
        =\left\| \left(f_{p, q}(\Psi(r_n, \cdot)), g_{p, q} (\Psi(s_n, \cdot)) \right) - \left(f_{p, q}(\Psi(r, \cdot)), g_{p, q} (\Psi(s, \cdot)) \right) \right\|_{\widetilde{C}(X) }
    \end{multline*}
    also goes to 0 as $n \to \infty$, as (upon fixing a representative of $(f_{p, q}, g_{p, q})$)
    \[r', s', x \mapsto f_{p, q}(\Psi(r', x)), g_{p, q}(\Psi(s', x)) \] 
    is uniformly continuous on the compact set $\overline{B_R(r)} \times \overline{B_R(s)} \times X$ for some small enough $R > 0$. Thus
    \[ \left\| H(p_n, q_n, r_n, s_n) - H(p, q, r, s) \right\|_{\widetilde{C}(X) } \to 0 \]
    as $n \to \infty$.

    We now check that the first partial derivatives of $H$ exist and are continuous.
    For the partial derivatives of $H$ with respect to $p, q$, by \cref{lemma:schro-potentials-regularity}, $\widetilde{H}(p, q) \vc = (f_{p,q}, g_{p, q})$ is a $C^k$ map from $U \times U$ to $\widetilde{C}^k(X)$ (and hence $\widetilde{C}(X)$). For each $p, q \in U \times U$, let $\Lambda_{p, q}\vc = D\widetilde{H}(p, q) \in \mathcal{B}(\R^m \times \R^m, \widetilde{C}^k(X) ).$
    Thus for any $r, s \in U$,
    \[ \frac{
    \left\| H(p+\eps_p, q+ \eps_q, r, s) - H(p, q, r, s) - \Lambda_{p, q}(\eps_p, \eps_q) \circ (\Psi(r, \cdot), \Psi(s, \cdot)) \right\|_{\widetilde{C}(X) }
    }{\|(\eps_p, \eps_q) \|} \to 0
    \]
    as $(\eps_p, \eps_q) \to 0$. Hence for each $p, q, r, s \in U,$ the partial derivative $D_{p, q}H(p, q, r, s) \in \mathcal{B}(\R^m \times \R^m, \widetilde{C}(X))$ exists and is given by
    \[ D_{p, q}H(p, q, r, s) = \left( \eps_p, \eps_q, \mapsto \Lambda_{p, q}(\eps_p, \eps_q) \circ \left( \Psi(r, \cdot), \Psi(s, \cdot) \right)  \right). \]
    As $p, q \mapsto \Lambda_{p, q}$ is a $C^{k-1}$ map from $U \times U$ to $\mathcal{B}(\R^m \times \R^m, \widetilde{C}^k(X))$, the continuity of \[D_{p, q}H \colon U \times U \times U \times U \to \mathcal{B}(\R^m \times \R^m, \widetilde{C}(X)) \cong \left( \widetilde{C}(X) \right)^{2m}\]
    follows from a similar argument as the continuity of $H$. 

    For the partial derivative of $H$ with respect to $r$, let $\widetilde{f}_{p, q, x}(r) \vc = f_{p, q}(\Psi(r, x))$ for each $p, q \in U$ and $x \in X$. (By a slight abuse of notation, we have chosen an arbitrary representative for $f_{p, q}$.) Since one partial derivative of $\Psi$ with respect to the first variable exists and is continuous on $U \times X$, $\widetilde{f}_{p, q, x}$ is $C^1$ on $U$ for each $p, q \in U$ and $x \in X$, and $r, x \mapsto \nabla \widetilde{f}_{p, q, x}(r)$ is continuous on $U \times X$ for each $p, q \in U$. Fix any $p, q, r \in U$. By the mean value theorem, for any $\eps \in \R^m$ with small enough norm,
    \[ \widetilde{f}_{p, q, x}(r + \eps) - \widetilde{f}_{p, q, x}(r) = \nabla \widetilde{f}_{p, q, x}(r + \alpha \eps) \cdot \eps \]
    for some $\alpha \in (0, 1)$. 
    As $r', x \mapsto \nabla \widetilde{f}_{p, q, x}(r')$ is uniformly continuous on $\overline{B_R(r)} \times X$ for some small enough $R > 0$, we have 
   \[ \lim_{\eps \to 0} \frac{\sup_{x \in X} \left| \widetilde{f}_{p, q, x}(r+ \eps) - \widetilde{f}_{p, q, x}(r) - \nabla \widetilde{f}_{p, q, x}(r) \cdot \eps \right|}{\|\eps\|}  = 0. \]
   Hence for each $p, q, r, s\in U$, $D_rH(p, q, r, s) \in \mathcal{B}(\R^m, \widetilde{C}(X))$ exists and is given by
   \[ [D_rH(p, q, r, s)](\eps) = \left[ x \mapsto \left(\nabla \widetilde{f}_{p, q, x}(r) \cdot \eps , 0\right) \right]. \]
   To show that
   \[ D_r H \colon U \times U \times U \times U \to \mathcal{B}(\R^m, \widetilde{C}(X)) \cong \left( \widetilde{C}(X) \right)^m \]
    is continuous, 
    observe that \cref{lemma:schro-potentials-regularity} implies that $p, q \mapsto \nabla f_{p, q}$ is a $C^k$ map from $U \times U$ to $\left(C^{k-1}(X)\right)^d$. Using the assumptions on the regularity of $\Psi$, another uniform continuity argument then implies that
    \[ p, q, r \mapsto \left( x \mapsto \nabla \widetilde{f}_{p, q, x}(r) \right) \]
    is continuous from $U \times U \times U$ to $\left(C(X)\right)^d$, and the continuity of $D_r H$ follows. The arguments for $D_sH$ are completely analogous.
    
    If $k \geq 2$, the arguments for higher-order derivatives are similar. Indeed, for each $p,q, r, s\in U$, $D_{p, q} H(p, q, r, s)$ is given by pre-composing $\Lambda_{p,q}$ with $(\Psi(r, \cdot), \Psi(s, \cdot)),$ where by a slight abuse of notation we consider $\Lambda_{p, q}$ as an element of $( \widetilde{C}^k(X) )^{2m}$; now recall that $p, q \mapsto \Lambda_{p, q}$ is a $C^{k-1}$ map from $U \times U$ to $\mathcal{B}(\R^m \times \R^m, \widetilde{C}^k(X) ) \cong ( \widetilde{C}^k(X) )^{2m}$ by \cref{lemma:schro-potentials-regularity}. Similarly, for $D_r H$, if $k \geq 2$ one can prove the existence and continuity of additional derivatives with respect to $r$ (again via the chain rule, assumptions on the regularity of $\Psi$ and uniform continuity arguments) and with respect to $p, q$ using continuity of the partial derivatives (\ref{eq:psi-part-derivs}) and \cref{lemma:schro-potentials-regularity}.
\end{proof}

\section{Additional details and plots for \cref{sec:examples}}
\label{sec:example-app}
We first prove Claim \ref{claim:rotating-particle-closed-form}.

\begin{proof}[Proof of Claim \ref{claim:rotating-particle-closed-form}]
    Fix any $\theta \in [-\pi, \pi)$. By symmetry, we have that the optimal coupling is of the form
       \begin{equation}
       \label{eq:optimal-coupling-1}
       \Pi_\theta^* = \begin{bmatrix}
        \frac{1}{2}-p(\theta) & p(\theta) \\
        p(\theta) & \frac{1}{2} -p (\theta)
    \end{bmatrix}, 
    \end{equation}
    and it remains to prove that $p(\theta)$ is given by (\ref{eq:rotation-p-closed-form}). Since $\mu_0$ and $\mu_\theta$ are both discrete measures with finite support, we can write (\ref{eq:optimal-coupling-form}) as
    \begin{equation}
    \label{eq:optimal-coupling-2}
    \Pi_\theta^* = \frac{1}{4}\begin{bmatrix}
        a_1 b_1 & a_1 b_2 e^{-R^2/\beta}  \\
        a_2 b_1 e^{-R^2/\beta} & a_2 b_2 e^{-\left(2R \sin \left( \frac{\theta}{2} \right) \right)^2/\beta}
    \end{bmatrix} 
    \end{equation}
    where $a_1 = e^{f_1/\beta}, a_2 = e^{f_2 / \beta}, b_1 = e^{g_1/\beta}, b_2 = e^{g_2/\beta}$ for some $f_1, f_2, g_1, g_2 \in \R$.
    Without loss of generality, we can take $a_1 = b_1$, and since
    \[ a_1 b_2 e^{-R^2/\beta} = a_2 b_1 e^{-R^2/\beta} \]
    we also have $a_2 = b_2$. Therefore
    \[ \Pi_\theta^* = \begin{bmatrix}
        \frac{1}{2}-p(\theta) & p(\theta) \\
        p(\theta) & \frac{1}{2} -p (\theta)
    \end{bmatrix} = \frac{1}{4} \begin{bmatrix}
        a_1^2 & a_1 a_2 e^{-R^2/\beta}  \\
        a_1 a_2 e^{-R^2/\beta} & a_2^2 e^{-\left(2R \sin \left( \frac{\theta}{2} \right) \right)^2/\beta}
    \end{bmatrix} \]
    for some $a_1, a_2 > 0,$ so
    \[ \left( \frac{\frac{1}{2} - p(\theta)}{p(\theta)} \right)^2 = \frac{a_1^2 a_2^2 e^{-\left(2R \sin \left( \frac{\theta}{2} \right) \right)^2/\beta}}{\left( a_1 a_2 e^{-R^2/\beta} \right)^2} = e^{2R^2\left( 1-2\sin^2 \left( \frac{\theta}{2} \right)  \right)/\beta} = e^{2R^2 \cos(\theta)/\beta}. \]
    Taking the square root of both sides gives us
    \[ \frac{1}{2p(\theta)} - 1 = e^{R^2 \cos(\theta)/\beta}, \]
    which, upon rearranging, gives us (\ref{eq:rotation-p-closed-form}):
    \[ p(\theta) = \frac{1}{2 \left( 1 + e^{R^2\cos(\theta)/\beta} \right)}. \]
    (\ref{eq:ot-beta-closed-form}) is obtained by observing that 
    \[ e^{\frac{f_1 + g_1}{\beta}} = e^{\frac{f_2 + g_2 -\left(2R \sin \left( \frac{\theta}{2} \right) \right)^2}{\beta}} = 4 \left( \frac{1}{2} - p(\theta) \right)  =  2 - 4 p(\theta)  \]
    and the dual representation (\ref{eq:ent-reg-ot-dual}) of $\mathrm{OT}_\beta$ gives us
    \begin{align*}
        \mathrm{OT}_\beta(\mu_0, \mu_\theta) &= \frac{1}{2} \left( f_1 + f_2 + g_1 + g_2 \right) \\
        &= \frac{1}{2} \left( \beta \log(2-4p(\theta)) + \beta \log(2-4p(\theta)) + \left(2R \sin \left( \frac{\theta}{2} \right) \right)^2 \right) \\
        &= \beta \log(2-4p(\theta)) + R^2 \left( 1 - \cos(\theta) \right).
    \end{align*}
    The expression for $S_\beta(\mu_0, \mu_\theta)$ follows from the definition of $S_\beta$. Now $g_0^{(\beta)}$ can be computed directly by considering the second derivative of $H(\theta) \vc = S_\beta(\mu_0, \mu_\theta)$ evaluated at $\theta = 0$. Indeed,
    \begin{align*}
    p'(\theta) &= -\frac{1}{2 \left( 1 + e^{R^2\cos(\theta)/\beta} \right)^2} \left(e^{R^2\cos(\theta)/\beta} \right) \left( -\frac{R^2 \sin (\theta)}{\beta} \right) \\
    &= p(\theta) \left( \frac{e^{R^2\cos(\theta)/\beta}}{1 + e^{R^2\cos(\theta)/\beta}} \right) \left( \frac{R^2 \sin(\theta)}{\beta} \right) \\
    &= p(\theta) (1- 2p(\theta))\left( \frac{R^2 \sin(\theta)}{\beta} \right)
    \end{align*}
    for $\theta \in (-\pi, \pi)$,
    and so
    \begin{align*}
        H'(\theta) &= -\frac{\beta p'(\theta)}{\frac{1}{2}-p(\theta)} + R^2\sin(\theta) \\
        &= -2 p(\theta) \left( R^2 \sin(\theta)\right) + R^2 \sin(\theta) \\
        &= R^2 \sin (\theta) \left( 1 - 2p(\theta) \right)
    \end{align*}
    for $\theta \in (-\pi, \pi)$. Taking one more derivative gives us
    \[ g_0^{(\beta)}\left( \frac{d}{d\theta},  \frac{d}{d\theta}\right) = \frac{1}{2} H''(0) = \frac{1}{2}R^2(1-2p(0)).\]
    By rotational symmetry, we have $g_\theta^{(\beta)} = g_0^{(\beta)}$ for all $\theta \in [-\pi, \pi).$
\end{proof}
\begin{figure}
    \centering
    \includegraphics[width=0.7\linewidth]{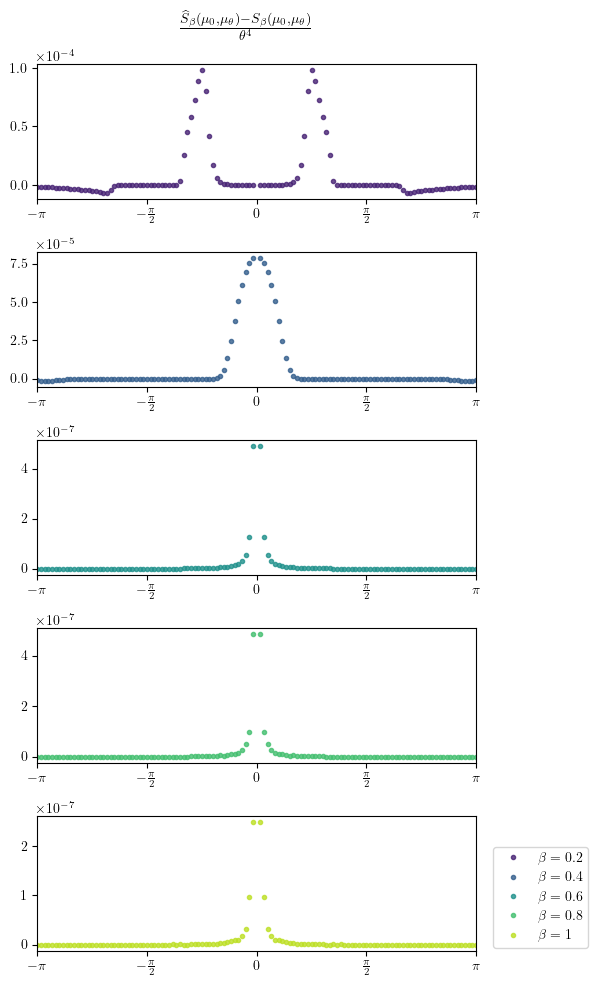}
    \caption{
    Plots of $\theta \mapsto \frac{\widehat{S}_\beta(\mu_0, \mu_\theta) - S_\beta(\mu_0, \mu_\theta)}{\theta^4}$, where
    $\widehat{S}_\beta$ is the Sinkhorn divergence approximated with Sinkhorn's algorithm and  $S_\beta(\mu_0, \mu_\theta)$ is given by the closed-form formula for the example in \cref{subsec:rotation}.}
    \label{fig:rotating-empirical-vs-closed}
\end{figure}

\begin{figure}
    \centering
    \includegraphics[width=0.75\linewidth]{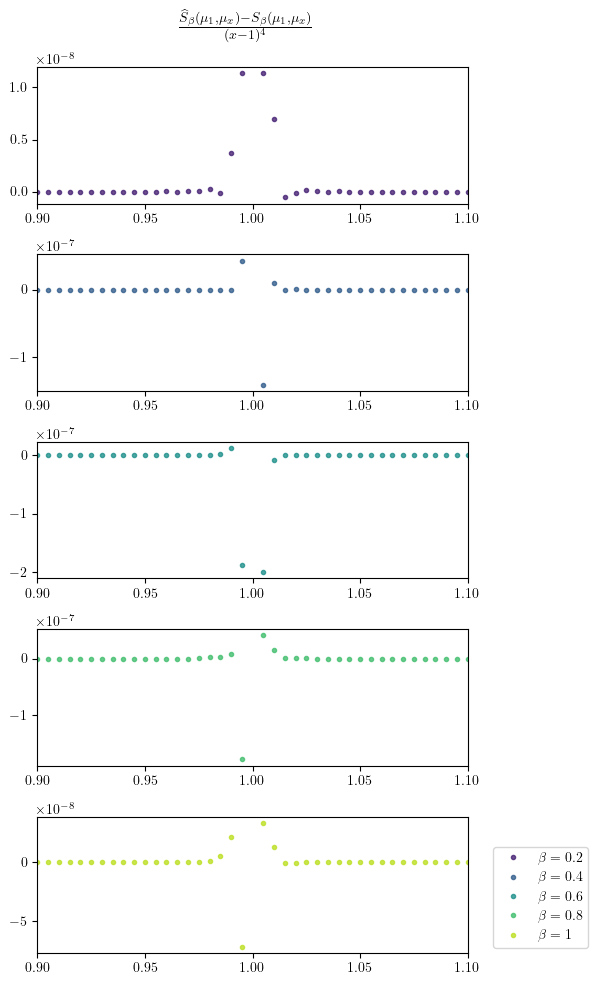}
    \caption{ Plots of $x\mapsto \frac{\widehat{S}_\beta(\mu_1, \mu_x) - S_\beta(\mu_1, \mu_x)}{(x-1)^4}$, where $\widehat{S}_\beta$ is the Sinkhorn divergence approximated with Sinkhorn's algorithm and  $S_\beta(\mu_1, \mu_x)$ is given by the closed-form formula for the example in \cref{subsec:dilation}.}
    \label{fig:dilation-empirical-vs-closed}
\end{figure}

\begin{figure}
    \centering
    \includegraphics[width=\linewidth]{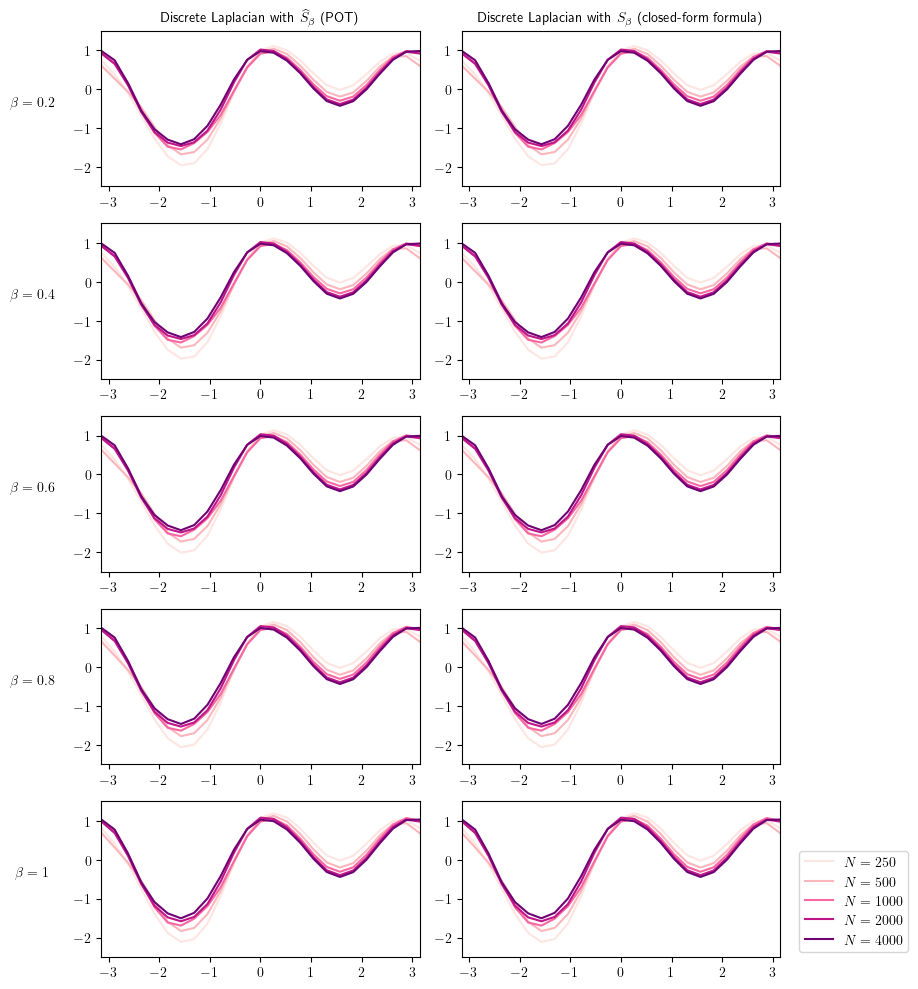}
    \caption{Same as \cref{fig:rotating-particle-gl-conv}, but with different values of $\beta$ ($0.2, 0.4, 0.6, 0.8, 1$) and using the Sinkhorn divergence $\widehat{S}_\beta$ approximated via Sinkhorn's algorithm.}
    \label{fig:rotating-particle-discrete-laplacians}
\end{figure}

\begin{figure}
    \centering
    \includegraphics[width=.9\linewidth]{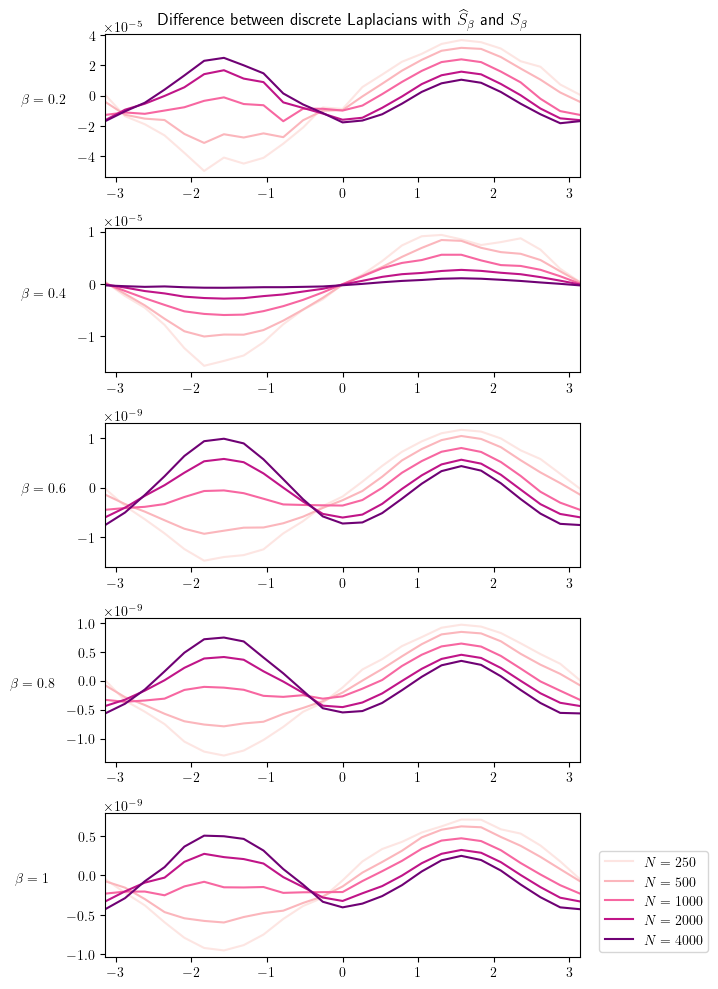}
    \caption{Subtracting the right-hand column from the left-hand column of \cref{fig:rotating-particle-discrete-laplacians}.}
    \label{fig:rotating-particle-discrete-laplacian-difference}
\end{figure}

As noted in Remark \ref{rmk:sinkhorn-algo}, in practice we often do not have a closed-form formula for the Sinkhorn divergence between two probability measures, and instead we have to approximate it. For Figures \ref{fig:rotating-particle-4th-order-error} and \ref{fig:fourth-order-error} (demonstrating the approximation error on the squared geodesic distance arising from \cref{lemma:third-order-vanish}), we find that the difference between using the closed-form formula and the Sinkhorn divergence computed using the default implementation of Sinkhorn's algorithm in the Python Optimal Transport (POT) library \cite{flamary2021pot, flamary2024pot} is very small; see Figures \ref{fig:rotating-empirical-vs-closed} and 
\ref{fig:dilation-empirical-vs-closed}, respectively.

One therefore expects the difference between the resulting discrete Laplacians to also be very small. As an example, in Figure \ref{fig:rotating-particle-discrete-laplacians}, we recreate Figure \ref{fig:rotating-particle-gl-conv} but with the Sinkhorn divergence computed using the default implementation of Sinkhorn's algorithm in POT and different values of $\beta$ ($0.2, 0.4, 0.6, 0.8, 1$). The difference between using the closed-form formula and the POT version is virtually unnoticeable unless we specifically plot the difference between the resulting discrete Laplacians (Figure \ref{fig:rotating-particle-discrete-laplacian-difference}). Observe that for each $\beta \in \{0.2, 0.4, 0.6, 0.8, 1\},$ the scale on the $y$-axis in Figure \ref{fig:rotating-particle-discrete-laplacian-difference} is much smaller than that on the corresponding plot in Figure \ref{fig:rotating-particle-discrete-laplacians}.

\end{document}